%% file: main.tex
\documentclass{article}
\usepackage{flowcred_preprint,times}
\usepackage{amsmath,amssymb,mathtools,amsthm}
\usepackage{booktabs,multirow,makecell,tabularx,array}
\usepackage{graphicx,subcaption}
\graphicspath{{Figures/}}
\usepackage{xcolor,colortbl}
\usepackage{algorithm}
\usepackage{algpseudocode}
\usepackage{microtype}
\usepackage{enumitem}
\usepackage{siunitx}
\usepackage{placeins}
\usepackage{hyperref}
\usepackage{url}
\hypersetup{colorlinks=true,linkcolor=black,citecolor=black,urlcolor=black,pdftitle={FlowCred: Hierarchical Credit Assignment and Functional Coordination for Joint Text-to-Video Post-Training},pdfauthor={Mingxuan Zhu, Yanwei Huang, Shiyuan Liu, Xinzhi Zhang, Shujie Li, Yuanxing Zhang, Lu Chen}}
\definecolor{softgray}{gray}{0.970}
\definecolor{rowlight}{gray}{0.955}
\definecolor{rowmid}{gray}{0.920}
\definecolor{rowdark}{gray}{0.875}
\definecolor{groupgray}{gray}{0.945}
\definecolor{groupbg}{RGB}{241,244,247}
\definecolor{grouptext}{RGB}{66,78,91}
\newcommand{\method}{\textsc{FlowCred}}
\newcommand{\E}{\mathbb{E}}
\newcommand{\sg}[1]{\operatorname{sg}\!\left[#1\right]}
\newcommand{\best}[1]{\textbf{#1}}
\newcommand{\second}[1]{\underline{#1}}
\newcommand{\msd}[2]{#1\,{\scriptsize(#2)}}
\newcommand{\TtwoVC}{T2V-CompBench}
\newcommand{\VBclassic}{VBench}
\newcommand{\VBtwo}{VBench-2.0}
\newtheorem{proposition}{Proposition}

\title{FlowCred: Hierarchical Credit Assignment and Functional Coordination for Joint Text-to-Video Post-Training}
\author{%
\parbox{0.97\textwidth}{\centering
Mingxuan Zhu$^{1,2}$\thanks{This work was conducted during the author's internship at KlingAI.}\quad Yanwei Huang$^{2,3}$\quad Shiyuan Liu$^{2}$\\[3pt]
Xinzhi Zhang$^{2,4}$\quad Shujie Li$^{2,5}$\quad Yuanxing Zhang$^{2}$\quad Lu Chen$^{1}$\\[9pt]
{\small $^{1}$Shanghai Jiao Tong University \quad $^{2}$KlingAI\\[2pt]
$^{3}$The Hong Kong University of Science and Technology\\[2pt]
$^{4}$South China University of Technology \quad $^{5}$The University of Hong Kong}\\[5pt]
}}

\begin{document}
\raggedbottom
\maketitle

\begin{abstract}
Prompt optimization and generator post-training are two complementary ways to improve text-to-video (T2V) generation, and recent systems increasingly train both modules jointly. However, optimizing a prompt policy and a video generator from the same terminal reward creates two ambiguities: prompt quality is confounded with stochastic generation, and the two modules may realize overlapping local improvements. We introduce \method{}, a two-level framework for joint T2V post-training. First, hierarchical credit assignment uses matched prompt--seed rollout grids to separate prompt-level credit from prompt-conditioned generator residuals under shared randomness. Second, functional coordination probes alternative current-policy prompts at states from actual generation trajectories, providing generator-aware evidence for prompt optimization while selectively attenuating generator updates already supported by conditioning changes. Across multiple T2V benchmarks, \method{} consistently improves matched joint-training baselines. On LTX-2.5, \method{} improves a strictly matched joint-training control from 61.47 to 64.46 on T2V-CompBench and from 55.58 to 58.08 on StoryEval, without increasing decoded-video or reward-model-call budgets. Factorial and realized-update analyses show complementary contributions from prompt shaping and generator routing, while human evaluation and transfer across backbones and updater families provide further supporting evidence.
\end{abstract}

\section{Introduction}
Text-to-video (T2V) systems increasingly expose two trainable adaptation channels: a language-side prompt policy and the video generator itself. Prompt-side methods improve conditioning while leaving the visual model fixed~\citep{hao2023promptist,cao2023beautifulprompt,ji2025promptavideo,cheng2025vpo}; generator post-training instead changes the diffusion/flow model using reward or preference feedback~\citep{black2023ddpo,fan2023dpok,wallace2024diffusiondpo,liu2025flowgrpo,xue2025dancegrpo,xue2025awm}. Recent systems optimize both modules jointly~\citep{wang2026promptrl,cheng2026vans,liu2026unigrpo,jia2026cape}. Prompt changes can reorganize entities, relations, and event descriptions before generation, whereas generator updates address residual motion, appearance, and execution errors that conditioning alone cannot remove. When only one module is updated, a terminal reward can be treated as supervision for that module; once both are trainable, the same scalar signal must implicitly serve both attribution and coordination roles. Joint post-training is therefore attractive; once both channels are trainable, however, one terminal score leaves two distinct questions unresolved.

The first is \emph{statistical credit}. A terminal reward conflates prompt quality with downstream generator randomness: identical prompts can score differently across seeds, while favorable trajectories can mask weaker prompt choices. The reward alone therefore does not identify which variation belongs to prompt choice and which remains after the prompt is fixed. For request $u$, prompt action $p$, stochastic generator trajectory $\mathcal T$, and terminal reward $R$, let $Q(u,p)=\E[R\mid u,p]$ and $V(u)=\E_pQ(u,p)$. Then
\begin{equation}
R-V(u)=\underbrace{Q(u,p)-V(u)}_{\text{prompt decision}}+\underbrace{R-Q(u,p)}_{\text{prompt-fixed trajectory}}.
\label{eq:intro-decomp}
\end{equation}
This exposes two sources of terminal variation. The second question is \emph{functional coordination}: even after statistical credit is separated, changing the prompt at a generator state may induce a response aligned with the same reward-improving direction used by the downstream optimizer. The two modules can therefore duplicate part of the same local improvement. Statistical credit and functional overlap are thus distinct problems.

These two ambiguities require different remedies: variance reduction can make credit estimates cleaner without preventing functional overlap, while coordinating local updates does not by itself determine where terminal variation originated. Treating them separately lets joint post-training distinguish \emph{who should receive credit} from \emph{which module should realize an overlapping improvement}.

To address both, we introduce \method{}, named for its joint treatment of credit assignment and coordination across the prompt-to-video generation flow. Figure~\ref{fig:overview} summarizes the two ambiguities and their corresponding mechanisms. \textbf{Hierarchical credit assignment} uses matched prompt--seed rollouts to separate prompt-level credit from prompt-conditioned generator residuals. \textbf{Functional coordination} probes alternative prompts at actual rollout states and uses same-state response geometry for generator-aware prompt shaping and reliability-gated generator routing.

\begin{figure*}[t]
\centering
\includegraphics[width=0.978\textwidth]{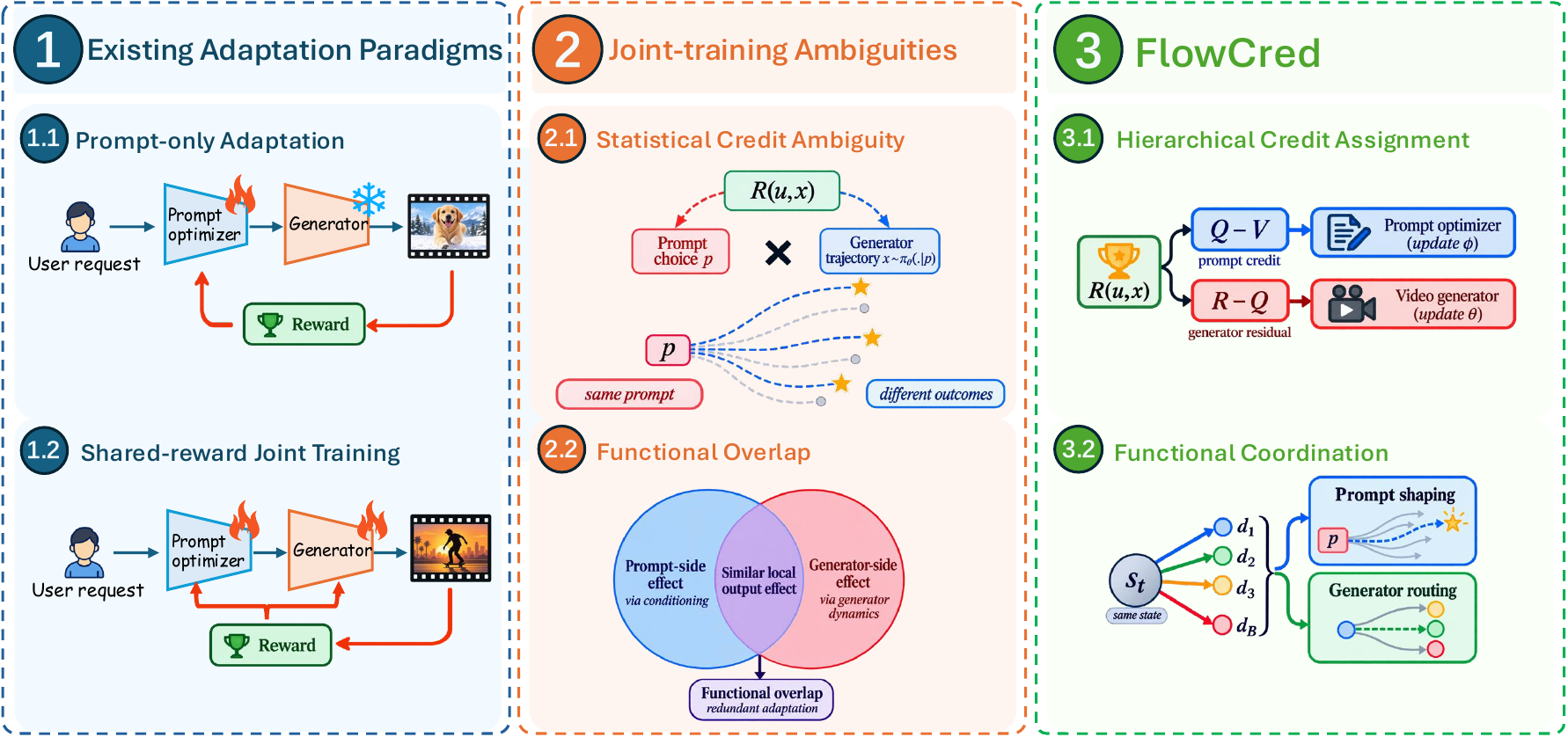}
\caption{\textbf{Motivation and positioning of \method{}.} The left panel contrasts prompt-only adaptation with shared-reward joint training. Once both modules are trained from a common terminal reward, two distinct issues arise: statistical credit ambiguity between prompt choice and prompt-fixed stochastic variation, and functional overlap between conditioning-side and generator-side local effects. \method{} addresses them with hierarchical credit assignment ($Q-V$ versus $R-Q$) and functional coordination based on same-state response geometry.}
\label{fig:overview}
\end{figure*}

The two mechanisms have different mathematical status: Eq.~\ref{eq:intro-decomp} is exact for any per-trajectory scalar score, whereas functional coordination is local and first-order. Its prompt term is auxiliary shaping and its generator rule is a minimum-overlap regularizer, not a conserved transfer of gradient mass. To keep the statistical claim separate from optimization heuristics, all affine calibration used to combine reward credit and functional prompt-credit terms is estimated once on disjoint training-only data and then frozen (Appendix~\ref{app:reward-recipe}).

Our contributions are threefold.\par
\noindent\textbf{(1)} We introduce hierarchical credit assignment with matched prompt--seed rollouts, separating prompt-level credit from prompt-conditioned generator residuals under shared generator randomness. This gives the prompt policy and generator distinct estimators tied to their positions in the stochastic hierarchy.\par

\noindent\textbf{(2)} We develop functional coordination from same-state prompt responses, using local response geometry for generator-aware prompt shaping and reliability-gated routing of the generator reward update. It targets overlap that remains even after statistical credit has been separated.\par

\noindent\textbf{(3)} We evaluate \method{} with controlled joint-training baselines, mechanism audits, repeated runs, human evaluation, and transfer across backbones and updater families. Across four T2V benchmarks, \method{} consistently improves the matched controls without increasing decoded-video or reward-model-call budgets.

\section{Related Work}
\paragraph{Prompt optimization and generator post-training.}
Prompt-side methods optimize conditioning with a fixed generator~\citep{hao2023promptist,cao2023beautifulprompt,ji2025promptavideo,cheng2025vpo}, whereas PromptRL, VANS, UniGRPO, and CAPE-T2V couple language- and generator-side adaptation in different ways~\citep{wang2026promptrl,cheng2026vans,liu2026unigrpo,jia2026cape}. Generator alignment spans policy gradients, preference objectives, and differentiable reward optimization~\citep{black2023ddpo,fan2023dpok,wallace2024diffusiondpo,clark2024draft,li2024diffusionkto,xu2023imagereward,kirstain2023pickapic,liu2025flowgrpo,choi2026rvm,prabhudesai2024videoalign,wang2026diffusiondrf}; recent methods further refine reward shaping, trajectory allocation, and temporal credit~\citep{xue2025dancegrpo,xue2025awm,luo2026gcpo,zhu2026diffusionapo,ye2026shift,pu2026noiseasier,shen2026star,guan2026latentregisters,savani2026stepwise,oba2026dispo}. These directions are complementary to our cross-module question. Modern open T2V systems span diverse backbone designs~\citep{chen2024videocrafter2,yang2025cogvideox,kong2024hunyuanvideo,wan2025wan}; our experiments hold the downstream harness fixed whenever numerical comparisons require matched conditions.

\paragraph{Credit assignment, variance reduction, and functional coordination.}
COMA counterfactual credit holds other agents fixed while marginalizing one action~\citep{foerster2018coma}, while leave-one-out baselines and common random numbers are standard tools for variance reduction~\citep{ahmadian2024rloo,ng2000pegasus}. PCGrad/CAGrad manipulate conflicting parameter gradients across objectives~\citep{yu2020pcgrad,liu2021cagrad}. \method{} instead separates statistical credit across the stochastic prompt-policy/generator hierarchy and measures functional overlap through same-state conditioning responses in the generator's native output space.

\section{Method}
\label{sec:method}
\subsection{Joint post-training: credit and coordination}
Let $u\sim\mathcal D$ be a request, $p\sim\pi_\phi(\cdot\mid u)$ a prompt action, and $\mathcal T\sim\pi_\theta(\cdot\mid p)$ a stochastic generator trajectory ending in video $x$. The shared objective is
\begin{equation}
J(\phi,\theta)=\E_{u,p,\mathcal T}[R(u,x)].
\label{eq:joint-objective}
\end{equation}
All rewards are evaluated against the original request $u$. We use \emph{statistical credit} for how terminal variation is assigned among sampled decisions and \emph{functional coordination} for how locally overlapping conditioning and generator changes are handled. Accordingly, hierarchical credit assigns terminal variation according to where it arises, while functional coordination acts only on local function changes that both trainable modules can express. We refer to $\pi_\phi$ as the \emph{prompt optimizer} or \emph{prompt policy}.

\subsection{Hierarchical credit assignment}
\label{sec:level1}
Define $Q_\theta(u,p)=\E_{\mathcal T\sim\pi_\theta(\cdot\mid p)}[R(u,x)]$ and $V_{\phi,\theta}(u)=\E_{p\sim\pi_\phi(\cdot\mid u)}[Q_\theta(u,p)]$; we suppress parameter subscripts below. Thus $R-V=(Q-V)+(R-Q)$ separates prompt choice from prompt-fixed stochasticity. For each request, sample $K$ candidates $p_{1:K}\sim\pi_\phi(\cdot\mid u)$ and reuse the same $M$ generator-randomness seeds $z_{1:M}$ for every candidate:
\begin{equation}
x_{k,m},\mathcal T_{k,m}=G_\theta(p_k;z_m),\qquad R_{k,m}=R(u,x_{k,m}).
\label{eq:grid}
\end{equation}
The matched rollout grid yields the paired leave-one-out estimators
\begin{equation}
\begin{aligned}
\widehat A^{P,R}_k&=\frac1M\sum_m\left(R_{k,m}-\frac1{K-1}\sum_{j\neq k}R_{j,m}\right),\\
\widehat A^G_{k,m}&=R_{k,m}-\frac1{M-1}\sum_{m'\neq m}R_{k,m'}.
\end{aligned}
\label{eq:hierarchical-estimators}
\end{equation}
The first contrast compares prompt decisions under common generator randomness and estimates $Q-V$; the second compares trajectories after the prompt is fixed and estimates $R-Q$. Positive same-seed reward covariance lowers cross-prompt contrast variance (Appendix~\ref{app:theory}). The prompt optimizer uses $\widehat A^{P,R}$, while the generator reward branch uses $\widehat A^G$. Reward-component calibration is estimated on disjoint training-only data and frozen, so no statistic from the current $K\!\times\!M$ grid enters $R_{k,m}$ (Appendix~\ref{app:reward-recipe}).

\paragraph{Controlled global-LOO baseline.}
Matched Global-LOO uses the identical $K\!\times\!M$ grid but gives both modules one global leave-one-out trajectory advantage (Appendix~\ref{app:naive-relation}). Averaged over the $M$ trajectories for a prompt, this control preserves the same coarse prompt ordering up to positive scale. The substantive statistical change in hierarchical credit is therefore downstream: generator credit becomes the prompt-conditional residual $R-Q$ rather than a global outcome residual.

\subsection{Functional coordination via same-state responses}
\label{sec:level2}
Hierarchical credit separates endpoint variation but does not describe whether prompt conditioning and the generator reward step realize similar local function changes. We measure that overlap at states actually encountered by the generator and reuse the same response geometry for prompt shaping and generator routing. An encountered state $s_t^{k,m}$ contains the noisy latent, timestep, and scheduler quantities required for one transition on trajectory $(k,m)$. Let $\mu_t=\mu_\theta(s_t^{k,m},p_k)$ be the backend-native transition output. For Full, the base updater's reward branch on this trajectory is weighted by $\widehat A^G_{k,m}$ from Eq.~\ref{eq:hierarchical-estimators}; stabilization/reference terms remain in a separate branch:
\begin{equation}
J_{\mathrm{base},t}=J_{\mathrm{reward},t}+J_{\mathrm{reg},t},\qquad
q_t^R=\nabla_{\mu_t}J_{\mathrm{reward},t},\quad g_t^R=F_t^{-1}q_t^R.
\label{eq:reward-dir}
\end{equation}
Here $F_t$ is the backend's local metric (the transition Fisher for Flow-GRPO), $q_t^R$ is the reward covector, and $g_t^R$ is its metric-mapped local ascent direction; we write $\|v\|_{F_t}^2=v^\top F_tv$. Reference/KL regularization does not enter $q_t^R$ and is restored unchanged after routing.

We cache the selected state rather than an alternative continuation. With the detached behavior snapshot $\theta^-=\operatorname{sg}(\theta)$, replacing only the executed prompt $p_k$ by another current-policy candidate $p_j$ gives
\begin{align}
d_{j\leftarrow k,t}&=\mu_{\theta^-}(s_t^{k,m},p_j)-\mu_{\theta^-}(s_t^{k,m},p_k),\nonumber\\
\bar d_{j\leftarrow k,t}&=\operatorname{TrustClip}_{F_t}(d_{j\leftarrow k,t}),\qquad
c_{j\leftarrow k,t}=\bar d_{j\leftarrow k,t}^{\top}q_t^R.
\label{eq:response}
\end{align}
The probe is intentionally one-transition and same-state: $s_t^{k,m}$ is seed/state dependent, but within each comparison the state, stochastic draw, and behavior snapshot are fixed and only text conditioning changes. Continuing $p_j$ would mix this local conditioning response with later nonlinear state evolution and additional stochasticity. We use
$\operatorname{TrustClip}_{F_t}(d)=d\min\!\left(1,\sqrt{\tau_{\rm batch}/(d^\top F_t d+\epsilon_{\rm num})}\right)$,
where $\epsilon_{\rm num}=10^{-12}$ and $\tau_{\rm batch}$ is the detached 95th percentile of the current-batch native-metric response energy $d^\top F_td$. Thus in-region responses are unchanged and only unusually large conditioning shifts are rescaled. The score $c_{j\leftarrow k,t}$ is the resulting local directional derivative along an observed prompt response; its sign records agreement with the reward branch. Full-trajectory utility of prompt candidates is still measured by the $K\!\times\!M$ rollouts in Eq.~\ref{eq:grid}.

\begin{figure*}[t]
\centering
\includegraphics[width=0.985\textwidth]{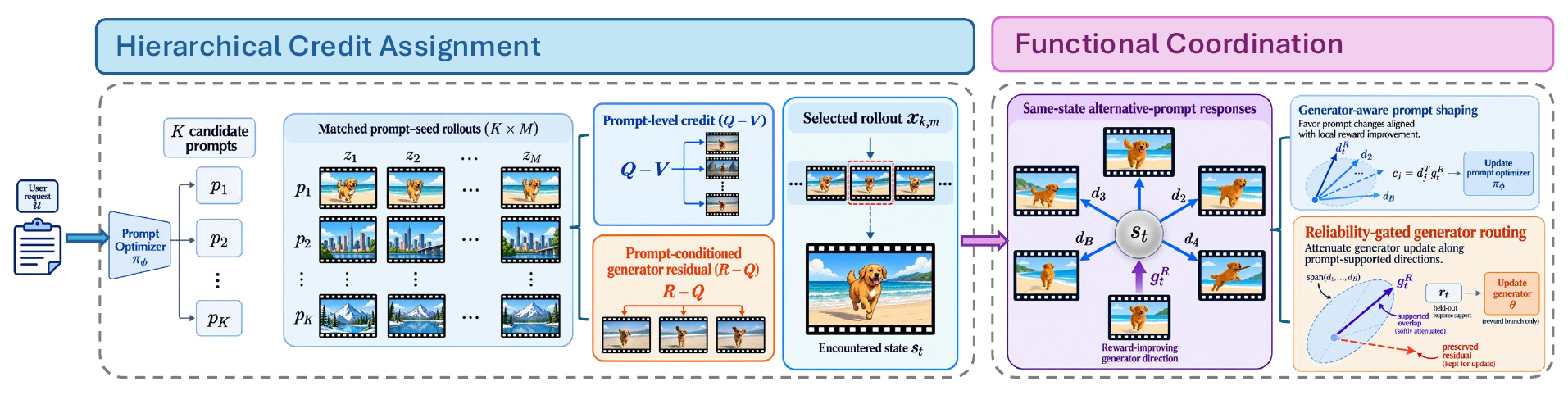}
\caption{\textbf{\method{} training overview.} Hierarchical credit assignment samples $K$ prompt candidates and reuses $M$ matched generator seeds, so cross-prompt contrasts estimate prompt-level credit $Q-V$ while within-prompt residuals estimate $R-Q$. A selected rollout supplies an encountered state $s_t$ for functional coordination. At that fixed state, alternative-prompt responses define a local geometry used for generator-aware prompt shaping and reliability-gated routing of only the generator reward branch.}
\label{fig:geometry}
\end{figure*}

\paragraph{Generator-aware prompt shaping.}
Let $\mathcal S_{k,m}$ denote the selected routed transition indices on trajectory $(k,m)$ (one by default). Cross-fitting avoids scoring $p_j$ on states produced by itself:
\begin{equation}
C_j=\frac1{K-1}\sum_{k\neq j}\E_{m,\,t\in\mathcal S_{k,m}}[c_{j\leftarrow k,t}],\qquad
A_j^{P}=\operatorname{Cal}_{R}(\widehat A_j^{P,R})+\lambda_F\operatorname{Cal}_{C}(C_j).
\label{eq:pe-final}
\end{equation}
For $X\in\{R,C\}$, $\operatorname{Cal}_X(x)=(x-\mu_X^{\rm cal})/(\sigma_X^{\rm cal}+10^{-6})$ uses statistics estimated once on a disjoint training-only calibration split and then frozen; $\operatorname{Cal}_R$ acts on terminal prompt credit and $\operatorname{Cal}_C$ on functional evidence. The terminal term remains primary, while $C_j$ supplies generator-aware auxiliary shaping; $\lambda_F$ is fixed on a disjoint development split.

\paragraph{Reliability-gated generator routing.}
Let $D_t=[\bar d_{j\leftarrow k,t}]_{j\neq k}$. The sampled current-policy response operator is a trace-normalized ridge map. To estimate whether this finite span generalizes beyond the $K-1$ rollout candidates, we sample $B_{\rm train}$ additional intent-preserving prompt alternatives at the same state. Each requires only one same-state generator transition forward: no video is decoded and no reward model is called. With response $\widetilde d_b$, define
\begin{equation}
\begin{aligned}
P_{\lambda,t}&=D_t(D_t^\top F_tD_t+\lambda_tI)^{-1}D_t^\top F_t,
&\lambda_t&=\epsilon_{\rm ridge}\frac{\operatorname{tr}(D_t^\top F_tD_t)}{K-1},\\
r_t&=\operatorname{clip}\!\left(\frac1{B_{\rm train}}\sum_b
\frac{\|P_{\lambda,t}\widetilde d_b\|_{F_t}^2}{\|\widetilde d_b\|_{F_t}^2+\epsilon_{\rm num}},0,1\right),
&\alpha_t&=\alpha_{\max}r_t,\\
g_t^G&=(I-\alpha_tP_{\lambda,t})g_t^R,
&q_t^G&=F_tg_t^G.
\end{aligned}
\label{eq:routing}
\end{equation}
Here $\epsilon_{\rm ridge}$ is the fixed ridge scale, $\alpha_{\max}$ caps the routing strength, and $q_t^G$ is the routed reward covector backpropagated through the generator. The score $r_t$ measures held-out response support rather than reward confidence. Low-support states continuously recover the base reward step. For stochastic-transition updaters such as Flow-GRPO, $g_t^R$ is a stochastic estimate of the local reward gradient; the response operator and support gate are built from the cached state and detached prompt responses. Consequently the same linear routing applies to the conditional expected reward gradient, rather than implying that one high-dimensional transition sample itself has large prompt-supported energy. Appendix~\ref{app:transition-geometry} states this distinction explicitly. The same geometry therefore serves two targets: $C_j$ shapes the upstream prompt action, whereas Eq.~\ref{eq:routing} regularizes the downstream expected reward update without assuming an exact transfer between optimizers.

\subsection{Training procedure}
One transition per decoded trajectory is routed by default, sampled uniformly from the valid stochastic-policy steps used by the base reward loss; all other reward-loss transitions and the entire regularization/reference branch remain unchanged. Same-state quantities are stop-gradient targets from the detached behavior snapshot and are recomputed online each round. Default hyperparameters and backend-specific instantiations are in Appendices~\ref{app:implementation} and~\ref{app:transition-geometry}.

\begin{algorithm}[H]
\caption{One joint post-training round}
\label{alg:flowcred}
\small
\begin{algorithmic}[1]
\Require request $u$; prompt policy $\pi_\phi$; generator $G_\theta$; $K$ candidates; $M$ matched seeds
\Statex \textbf{Hierarchical credit assignment}
\State sample $p_{1:K}$ and decode $x_{k,m}=G_\theta(p_k;z_m)$ with shared $z_{1:M}$
\State evaluate $R_{k,m}$; compute $\widehat A^{P,R}_k$ and $\widehat A^G_{k,m}$
\Statex \textbf{Functional coordination at actual rollout states}
\For{each selected transition state $s_t^{k,m}$}
\State obtain $(q_t^R,g_t^R)$ from the hierarchy-weighted base reward branch
\State probe $d_{j\leftarrow k,t}$ for $j\neq k$ and accumulate cross-fitted evidence $C_j$
\State construct $P_{\lambda,t}$; estimate $r_t$ with $B_{\rm train}$ non-decoded prompt probes
\State route $g_t^G\leftarrow(I-\alpha_{\max}r_tP_{\lambda,t})g_t^R$
\EndFor
\State update prompt policy with Eq.~\ref{eq:pe-final}; update generator with the routed reward branch plus unchanged regularization
\end{algorithmic}
\end{algorithm}

\section{Experiments}
\label{sec:experiments}
\subsection{Protocol and controlled comparisons}
\paragraph{Setup.}
Our primary system uses LTX-2.5~\citep{hacohen2026ltx2,lightricks2026ltx25}, a Qwen3.5-9B prompt optimizer, Flow-GRPO, and $K=M=4$. Controlled variants share initialization, request stream, prompt sampling, matched seeds, reward, decoded-video budget, and optimization schedule; functional coordination adds same-state transition forwards but no extra video decode. The frozen training reward uses public, fixed scorers with weights $0.35/0.30/0.25/0.10$: the Text-Alignment head of VideoReward~\citep{liu2025videoreward} for request faithfulness, frame-averaged VQAScore (CLIP-FlanT5-XXL)~\citep{lin2024vqascore} for compositional correspondence, the temporal-consistency/text-alignment aggregate of VideoScore-v1.1~\citep{he2024videoscore} for event realization, and the visual-/motion-quality aggregate of VideoReward as a bounded quality stabilizer. All scorers and calibration statistics are frozen before controlled training. Appendices~\ref{app:implementation} and~\ref{app:reward-recipe} specify checkpoints, frame sampling, calibration, and scorer-family robustness.

\paragraph{Evaluation and comparisons.}
We report \VBclassic{}~\citep{huang2024vbench}, \VBtwo{}~\citep{zheng2025vbench2}, \TtwoVC{}~\citep{sun2025t2vcompbench}, and StoryEval~\citep{wang2025storyeval} on the original request $u$ with matched evaluation seeds. The hierarchical-credit-versus-Full comparison and the functional-coordination factorial use three independent training runs; other trainable baselines use two. The controlled chain spans prompt-only adaptation, a raw-request Flow-GRPO generator-only reference, PromptRL, UniGRPO, Matched Global-LOO, hierarchical credit, and Full. For methods originally reported with different backbones or training stacks, we implement the published objective in the common Qwen3.5/LTX-2.5 harness; the table therefore reports controlled common-harness results rather than copied published absolute scores. CAPE-T2V is positioned separately because its public LTX recipe uses sequential two-sided adaptation with additional video-derived caption supervision and a different released LTX backbone; Appendix~\ref{app:baseline-adaptation} gives the precise comparison scope. Uncertainty and data-integrity protocols are summarized in Appendices~\ref{app:evaluation},~\ref{app:integrity}, and~\ref{app:robustness}.

\subsection{Main benchmark performance}
Table~\ref{tab:main} compares single-side and joint alternatives under one LTX-2.5 harness. Because prompt-side adaptation already closes substantial interface-level headroom, the matched joint chain is the key test of hierarchical credit and functional coordination.

\begin{table}[H]
\caption{\textbf{Main LTX-2.5 comparison under a common evaluation harness.} Mean scores use matched evaluation seeds; training-run uncertainty is reported in Appendix~\ref{app:robustness}. Best is bold and second-best is underlined.}
\label{tab:main}
\centering
\small
\setlength{\tabcolsep}{4.6pt}
\renewcommand{\arraystretch}{1.02}
\begin{tabularx}{\linewidth}{@{}>{\raggedright\arraybackslash}X c c c c@{}}
\toprule
Method & \VBclassic{}$\uparrow$ & \VBtwo{}$\uparrow$ & \TtwoVC{}$\uparrow$ & StoryEval$\uparrow$ \\
\midrule
\rowcolor{groupbg}
\multicolumn{5}{@{\hspace{4pt}}l@{}}{\color{grouptext}\bfseries\small\textsc{Frozen / interface references}}\\[-0.5pt]
\cmidrule(lr){1-5}
Raw prompt & 78.14 & 54.52 & 46.52 & 37.88 \\
Official LTX-2.5 prompt enhancer & 84.92 & 61.36 & 57.48 & 51.63 \\
\addlinespace[2pt]
\rowcolor{groupbg}
\multicolumn{5}{@{\hspace{4pt}}l@{}}{\color{grouptext}\bfseries\small\textsc{Single-side adaptation}}\\[-0.5pt]
\cmidrule(lr){1-5}
Prompt-A-Video & 85.63 & 62.18 & 58.87 & 52.96 \\
VPO & 85.44 & \second{64.06} & 59.54 & 53.61 \\
Prompt-only & \second{86.97} & 62.84 & 60.28 & 54.37 \\
Flow-GRPO & 80.83 & 58.57 & 52.31 & 45.26 \\
\addlinespace[2pt]
\rowcolor{groupbg}
\multicolumn{5}{@{\hspace{4pt}}l@{}}{\color{grouptext}\bfseries\small\textsc{Joint post-training --- matched decoded-rollout budget}}\\[-0.5pt]
\cmidrule(lr){1-5}
PromptRL & 86.31 & 63.02 & 61.06 & 55.02 \\
UniGRPO & 86.40 & 63.28 & 61.32 & 55.40 \\
Matched Global-LOO & 86.49 & 63.41 & 61.47 & 55.58 \\
\method{} (hierarchical credit) & 86.91 & 63.92 & \second{62.63} & \second{56.71} \\
\rowcolor{rowlight}\textbf{\method{} (Full)} & \best{87.56} & \best{64.73} & \best{64.46} & \best{58.08} \\
\bottomrule
\end{tabularx}
\end{table}

Prompt-side adaptation already closes substantial interface-level headroom, making the matched joint chain the stricter test. In particular, Prompt-only is already competitive on VBench, whereas the value of joint adaptation is more visible on compositional and multi-event metrics: Full exceeds Prompt-only by +4.18 T2V-CompBench and +3.71 StoryEval points. Within the matched joint chain, PromptRL and UniGRPO remain below Matched Global-LOO; hierarchical credit then improves over Global-LOO by +0.42/+0.51/+1.16/+1.13 on VBench/VBench-2.0/T2V-CompBench/StoryEval, and Full adds another +0.65/+0.81/+1.83/+1.37. The hierarchical-credit-to-Full increment is positive in all three independent training runs on every headline benchmark. Appendix~\ref{app:robustness} reports run-level deltas and intervals, and Table~\ref{tab:full-breakdown} gives the aggregate breakdown.

\subsection{Hierarchical-credit diagnostics}
The two estimator choices are directly testable at fixed decoded-video budget. Conditioning the generator baseline on the executed prompt removes prompt-level nuisance variation from the downstream advantage, while matched seeds pair prompt candidates under common generator randomness. Table~\ref{tab:level1-main} reports held-out diagnostics; the corresponding distributions and protocol are in Appendix~\ref{app:level1-efficiency}. Together, these diagnostics support the intended separation of prompt-level and prompt-conditioned variation while keeping the decoded-rollout budget fixed.

\begin{table}[H]
\caption{\textbf{Hierarchical-credit diagnostics at matched rollout budget.} Lower variance is better; the final row keeps the decoded-video count fixed.}
\label{tab:level1-main}
\centering
\small
\setlength{\tabcolsep}{5.5pt}
\begin{tabular}{lccc}
\toprule
Diagnostic & Reference & Hierarchical design & Result \\
\midrule
Generator-advantage variance & Global LOO & Prompt-conditioned LOO & $-31\%$ \\
Prompt-credit variance & Independent seeds & Matched seeds & $-27\%$ \\
\TtwoVC{} & Independent seeds & Matched seeds & $62.21\rightarrow62.63$ \\
\bottomrule
\end{tabular}
\end{table}

\subsection{Functional-coordination attribution and mechanism diagnostics}
The two uses of the same-state geometry contribute independently, and their combination yields the strongest result. Table~\ref{tab:components-main} reports the three-run mean (SD) endpoint scores, while Fig.~\ref{fig:mechanism-evidence} summarizes the complementary endpoint-attribution and realized-update audits. Relative to hierarchical credit, prompt shaping alone improves VBench/StoryEval/T2V-CompBench by +0.29/+0.65/+0.77, generator routing alone by +0.38/+0.47/+1.04, and Full by +0.65/+1.37/+1.83.

\begin{table}[H]
\caption{\textbf{Functional-coordination factorial ablation.} Mean (SD) across the same three-run protocol used for the HC--Full comparison.}
\label{tab:components-main}
\centering
\small
\setlength{\tabcolsep}{5.2pt}
\renewcommand{\arraystretch}{1.02}
\begin{tabular}{ccccc}
\toprule
Prompt shaping & Generator routing & \VBclassic{}$\uparrow$ & StoryEval$\uparrow$ & \TtwoVC{}$\uparrow$ \\
\midrule
-- & -- & \msd{86.91}{0.14} & \msd{56.71}{0.27} & \msd{62.63}{0.32} \\
$\checkmark$ & -- & \msd{87.20}{0.15} & \msd{57.36}{0.28} & \msd{63.40}{0.34} \\
-- & $\checkmark$ & \msd{87.29}{0.13} & \msd{57.18}{0.26} & \msd{63.67}{0.31} \\
\rowcolor{rowlight}$\checkmark$ & $\checkmark$ & \textbf{\msd{87.56}{0.12}} & \textbf{\msd{58.08}{0.24}} & \textbf{\msd{64.46}{0.29}} \\
\bottomrule
\end{tabular}
\end{table}

\begin{figure}[H]
\centering
\begin{minipage}[t]{0.49\linewidth}
\centering
\includegraphics[width=0.99\linewidth]{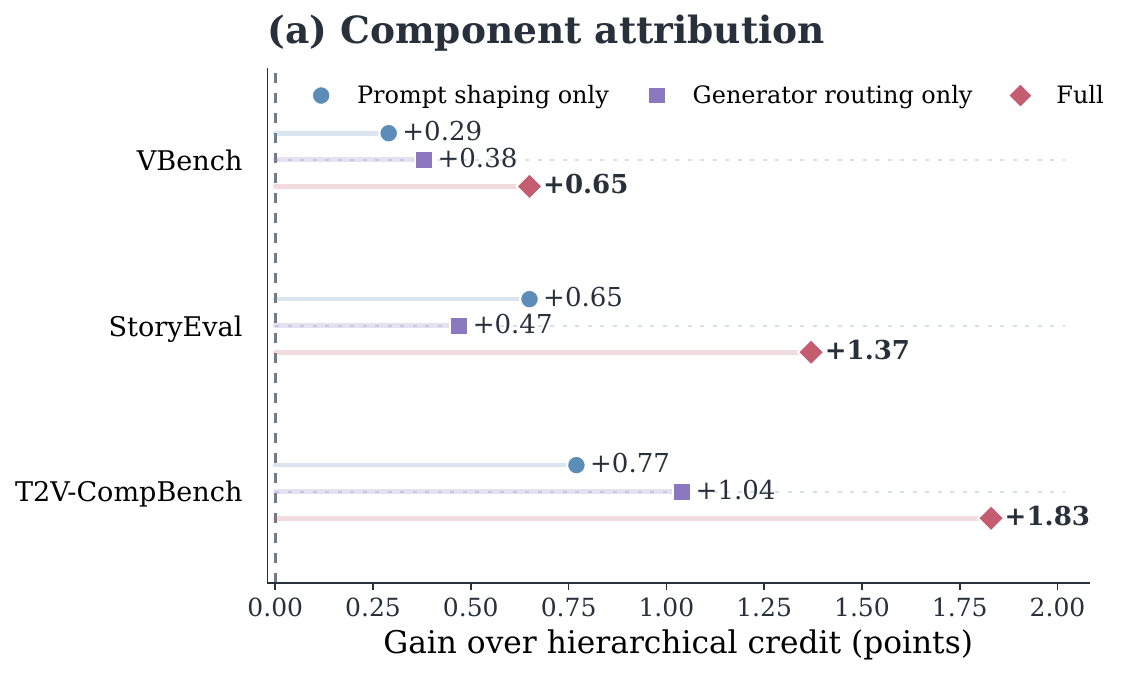}
\subcaption{Component attribution.}
\end{minipage}\hfill
\begin{minipage}[t]{0.49\linewidth}
\centering
\includegraphics[width=0.99\linewidth]{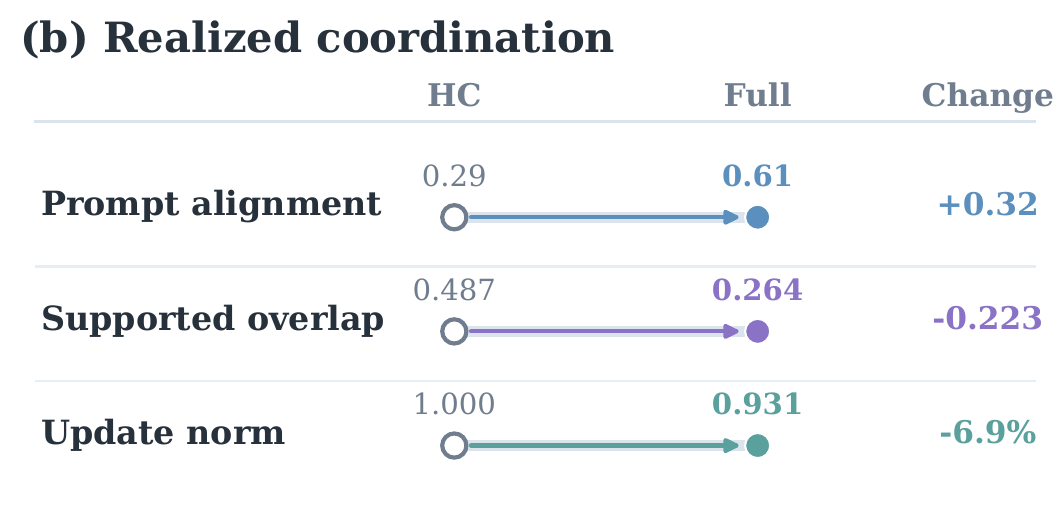}
\subcaption{Realized coordination.}
\end{minipage}
\caption{\textbf{Functional-coordination endpoint and mechanism evidence.} \textbf{(a)} Relative to hierarchical credit, prompt shaping and generator routing each improve endpoint performance, and their combination is strongest on all three metrics. \textbf{(b)} In a matched cached-state audit, Full increases prompt-side alignment, reduces supported overlap, and retains most of the overall generator update norm.}
\label{fig:mechanism-evidence}
\end{figure}

\paragraph{Realized coordination.}
The factorial result is complemented by a matched one-step audit on cached states. Rather than interpreting the projection energy of an individual stochastic policy-gradient sample, we measure the neural displacement after the optimizer step. Full reduces realized generator supported overlap from 0.487 to 0.264 while retaining 0.931$\times$ the hierarchical-credit update norm; on the prompt side, the native-metric alignment cosine rises from 0.29 to 0.61. Shaping-only and routing-only variants isolate the corresponding upstream and downstream changes (Appendix~\ref{app:realized-coordination}). Together with the endpoint factorial, these measurements support selective coordination rather than a generic shrinkage explanation, while remaining descriptive checks of trained behavior rather than an exact transfer of causal responsibility.
\FloatBarrier

\subsection{Prompt-policy signal validity and routing controls}
On 360 held-out requests, reward-only candidate ranking reaches $\rho=0.447$ / 59.2\% Top-1 on eight disjoint fresh seeds, cosine functional evidence reaches $\rho=0.569$ / 67.2\%, and the directional functional score used by \method{} reaches $\rho=0.612$ / 71.7\%. This provides an out-of-seed check that same-state generator information carries signal beyond terminal reward ranking alone.

Matched-budget controls test simpler routing explanations. Rank-matched random spans and wrong-query spans stay near the hierarchical-credit reference; uniform shrink yields only a small gain, while removing cross-fitting or the reliability gate reduces the Full improvement. Figure~\ref{fig:mechanism-controls} summarizes both checks; complete control values are reported in Appendix~\ref{app:mechanism-audit-full}.

\begin{figure}[H]
\centering
\begin{minipage}[t]{0.49\textwidth}
\centering
\includegraphics[width=0.99\linewidth]{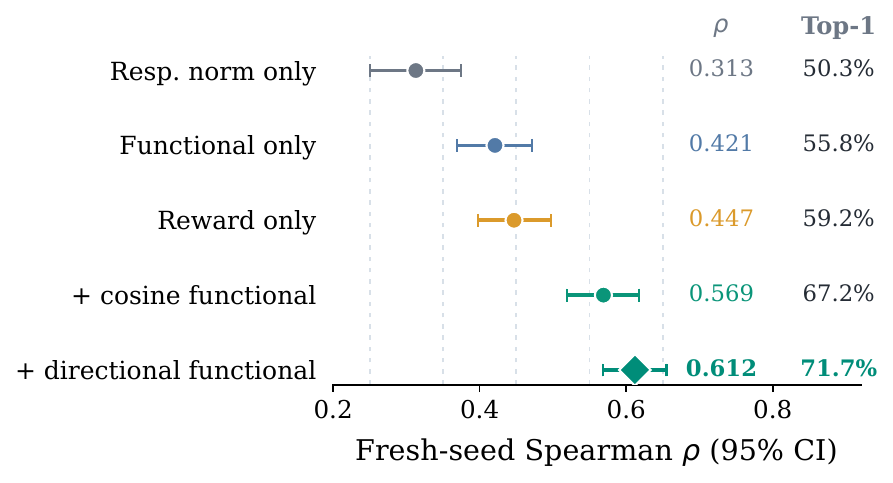}
\subcaption{Fresh-seed predictive utility.}
\end{minipage}\hfill
\begin{minipage}[t]{0.49\textwidth}
\centering
\includegraphics[width=0.99\linewidth]{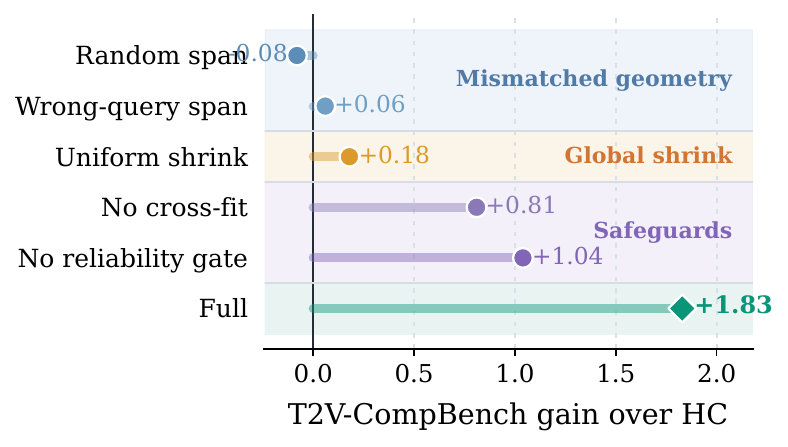}
\subcaption{Compact matched-budget routing controls.}
\end{minipage}
\caption{\textbf{Functional signal validity and routing controls.} Directional evidence improves fresh-seed candidate ranking beyond reward-only scoring, while compact routing controls test mismatched geometry and safeguard removal.}
\label{fig:mechanism-controls}
\end{figure}

\subsection{Generalization and external validation}
\noindent\textbf{Transfer evidence.} We transfer the functional-coordination increment to HunyuanVideo-1.5 and to RVM and Diffusion-DRF generator updates. Table~\ref{tab:transfer-main} reports matched Full-minus-hierarchical-credit deltas rather than heterogeneous absolute rankings.

\begin{table}[H]
\caption{\textbf{Transfer of functional coordination.} Full minus matched hierarchical credit under each backbone/updater.}
\label{tab:transfer-main}
\centering
\small
\setlength{\tabcolsep}{6.0pt}
\begin{tabular}{lccc}
\toprule
Setting & $\Delta$\VBclassic{} & $\Delta$\TtwoVC{} & $\Delta$StoryEval \\
\midrule
LTX-2.5 / Flow-GRPO & +0.65 & +1.83 & +1.37 \\
HunyuanVideo-1.5 / Flow-GRPO & +0.57 & +1.22 & +1.09 \\
LTX-2.5 / RVM & +0.52 & +1.29 & +1.15 \\
LTX-2.5 / Diffusion-DRF & +0.41 & +0.97 & +0.79 \\
\bottomrule
\end{tabular}
\end{table}

The increment remains positive in every tested backbone/updater setting, although its magnitude varies, indicating that functional coordination is not tied to a particular generator backbone or reward-update rule.

Full keeps the decoded-video and reward-model-call budget identical to hierarchical credit; its additional generator work consists of non-decoded same-state transition probes, counted explicitly in Appendix~\ref{app:compute}. Frozen-module swaps attribute measurable gains to both sides: substituting the Full generator into the hierarchical-credit pair adds +0.90 T2V-CompBench and substituting the Full prompt policy adds +0.72. In a 240-request blinded evaluation, Full is preferred over hierarchical credit on 134 versus 48 requests (58 ties; 73.6\% of non-ties) and over PromptRL on 142 versus 46 (52 ties; 75.5\% of non-ties). An independent VisionReward evaluator, never used in training or model selection, also prefers Full to hierarchical credit on 63.3\% of 300 held-out matched pairs. Detailed annotation and replication protocols are in Appendices~\ref{app:robustness} and~\ref{app:human}.

Taken together, these checks test complementary explanations of the primary gain. Transfer across a second backbone and alternative updater families probes dependence on the default generator stack; frozen-module swaps verify contributions from both trained modules; and blinded human judgments plus the held-out VisionReward evaluator test whether the improvement is confined to the training reward.

\noindent\textbf{External two-sided positioning.} PromptRL and UniGRPO use their published joint-training objectives in the common LTX-2.5 harness. CAPE-T2V is reported only as an external two-sided reference because its public pipeline uses a different released LTX backbone and auxiliary video-derived captions. Appendix~\ref{app:baseline-adaptation} details the comparison scope; Appendix~\ref{app:robustness} reports rollout-budget and paired-run checks.

\subsection{Qualitative coordination cases}
Figure~\ref{fig:qual-main} shows two representative cases with complementary demands. In the structured repair scene, Base and VPO capture the broad setting but do not make the requested spacecraft-repair interaction equally explicit, whereas \method{} preserves both the scene structure and the target interaction. In the directional-motion case, the generated sequence shows clearer execution of the requested right-to-left motion for \method{}, while the frame strip highlights its subject consistency; the alternatives show weaker directional execution or greater subject ambiguity. Appendix~\ref{app:qualitative} provides four additional cases spanning counting, interaction, spatial relation, and synchronized appearance change.

\begin{figure}[H]
\centering
\includegraphics[width=0.99\linewidth]{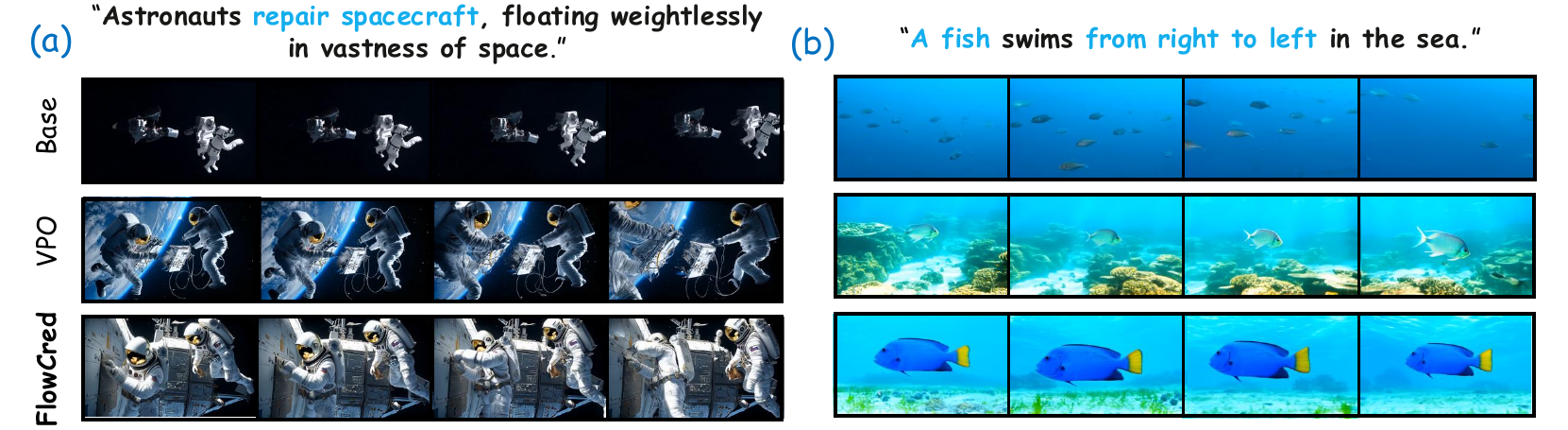}
\caption{\textbf{Representative qualitative comparison.} Base, VPO, and \method{} on structured scene realization and directional motion. Additional qualitative comparisons are provided in the appendix.}
\label{fig:qual-main}
\end{figure}

\section{Conclusion}
We presented \method{}, a framework for jointly post-training a prompt optimizer and a text-to-video generator without treating a shared terminal reward as undifferentiated supervision. Hierarchical credit assignment uses matched prompt--seed rollouts to separate prompt choice from prompt-fixed stochasticity, while functional coordination uses same-state response geometry for generator-aware prompt shaping and reliability-gated routing of the generator reward branch. Controlled comparisons show that the gains cannot be explained by stronger prompts, additional decoded rollouts, or generic update shrinkage alone, and evidence from repeated runs, mechanism analyses, human evaluation, and transfer studies is consistent with the intended division of labor between prompt-side conditioning and residual generator execution. More broadly, \method{} suggests that joint post-training should distinguish \emph{where outcome variation arises} from \emph{where trainable modules can realize overlapping functional changes}: sharing a reward is not sufficient by itself; credit assignment and update coordination are separate design problems in cascaded generative systems.

\section*{AI use statement}
Generative AI tools were used to assist with conceptual and experimental design, literature search, code development, and manuscript drafting and editing. All AI-assisted material was reviewed by the authors, who are responsible for the final submission.

\section*{Reproducibility statement}
Section~\ref{sec:method} specifies the \method{} formulation and training procedure; Appendices~\ref{app:theory}--\ref{app:human} provide derivations, implementation details, reward/calibration and baseline protocols, evaluation, compute, robustness, human evaluation, and qualitative results.

\begingroup
\small
\setlength{\bibsep}{2.0pt}
\bibliography{references}
\bibliographystyle{flowcred_preprint}
\endgroup

\clearpage
\appendix
\input{appendix.tex}
\end{document}

%% file: appendix.tex
\section{Theoretical Foundations}
\label{app:theory}
This appendix collects the mathematical statements used in the main paper. The reward hierarchy is exact for any per-trajectory scalar terminal score and yields valid score-function baselines. In the final protocol, reward-component and prompt-term affine scales are estimated once on disjoint training-only calibration data and then frozen, so current minibatch composition does not enter the score used by the hierarchy (Appendix~\ref{app:reward-recipe}). Same-state prompt intervention provides a first-order directional derivative of the generator reward surrogate. Minimum-overlap routing is a local function-space regularizer; its non-reversal statement is made in the backend's local output coordinates before neural pullback.

\subsection{Hierarchy-aware statistical credit}
\label{app:reward-theory}
Recall
\begin{equation}
J(\phi,\theta)=\E_{u\sim\mathcal D}\E_{p\sim\pi_\phi(\cdot|u)}\E_{\mathcal T\sim\pi_\theta(\cdot|p)}[R(u,x)],
\end{equation}
with
\begin{equation}
Q_\theta(u,p)=\E_{\mathcal T\sim\pi_\theta(\cdot|p)}[R(u,x)],\qquad
V_{\phi,\theta}(u)=\E_{p\sim\pi_\phi(\cdot|u)}Q_\theta(u,p).
\end{equation}

\begin{proposition}[Hierarchy-aware score-function baselines]
\label{prop:reward}
For the factorized policy $\pi_\phi(p|u)\pi_\theta(\mathcal T|p)$,
\begin{align}
\nabla_\phi J
&=\E[(Q_\theta(u,p)-V_{\phi,\theta}(u))\nabla_\phi\log\pi_\phi(p|u)],\\
\nabla_\theta J
&=\E[(R(u,x)-\sg{Q_\theta(u,p)})\nabla_\theta\log\pi_\theta(\mathcal T|p)].
\end{align}
\end{proposition}
\begin{proof}
For the prompt policy, apply the score-function identity and subtract $V(u)$, which does not depend on the sampled prompt action. For the generator, any prompt-dependent quantity that is fixed with respect to the sampled downstream trajectory is a valid score-function baseline; choosing $Q(u,p)$ gives the second expression. The stop-gradient indicates that the statement concerns a policy-gradient baseline rather than differentiation through the value estimate.
\end{proof}

The pointwise identity $R-V=(Q-V)+(R-Q)$ follows immediately. It is a statistical decomposition of reward variation. Functional coordination in Section~\ref{sec:level2} then shapes the joint objective using same-state response evidence, so the two levels have distinct mathematical status.

\subsection{Relation to the matched global-LOO control}
\label{app:naive-relation}
The matched control uses the same sampled actions and matched rollout grid as hierarchical credit, but both modules receive a global leave-one-out advantage
\begin{equation}
A^{\rm GLOO}_{k,m}=R_{k,m}-\frac{1}{KM-1}\sum_{(j,n)\neq(k,m)}R_{j,n}.
\end{equation}
For a fixed prompt action $p_k$, averaging over its $M$ trajectories yields a positive affine rescaling of $Q_k-\bar Q$ and therefore preserves the coarse prompt ordering. The hierarchical-credit prompt-policy estimator in Eq.~\ref{eq:hierarchical-estimators} is likewise a positive rescaling of $Q_k-\bar Q$. We consequently do not attribute the hierarchical-credit gain to a new prompt ordering. The controlled change is the downstream generator baseline: Global-LOO removes only the current trajectory, whereas hierarchical credit removes prompt-level variation as well by comparing the trajectory against other seeds under the same executed prompt. The global-LOO construction also eliminates the self-inclusion confound that would otherwise accompany a global mean baseline.

\subsection{Variance reduction from matched seeds}
For candidates $p_k,p_j$, write $R_k(z)$ and $R_j(z)$ for rewards obtained with common seed $z$. Define
\begin{equation}
\widehat\Delta_{kj}^{\rm match}=\frac1M\sum_{m=1}^{M}[R_k(z_m)-R_j(z_m)].
\end{equation}
For an independently seeded comparison, use $z_m$ for $k$ and independent $z_m'$ for $j$. Let $\sigma_k^2=\operatorname{Var}(R_k)$, $\sigma_j^2=\operatorname{Var}(R_j)$, and $\sigma_{kj}=\operatorname{Cov}(R_k(z),R_j(z))$.

\begin{proposition}[Matched-seed variance]
Assume independent seed pairs across $m$ and finite second moments. Then
\begin{align}
\operatorname{Var}(\widehat\Delta_{kj}^{\rm match})&=\frac{\sigma_k^2+\sigma_j^2-2\sigma_{kj}}{M},\\
\operatorname{Var}(\widehat\Delta_{kj}^{\rm ind})&=\frac{\sigma_k^2+\sigma_j^2}{M}.
\end{align}
Hence matched seeds reduce variance by $2\sigma_{kj}/M$ whenever $\sigma_{kj}\ge0$.
\end{proposition}
\begin{proof}
Apply the variance identity to $R_k(z)-R_j(z)$ and use independence across $m$. The independent estimator has zero cross-covariance by construction.
\end{proof}
The prompt-policy estimator in Eq.~\ref{eq:hierarchical-estimators} averages these seed-wise contrasts over the other $K-1$ candidates, so common randomness enters the actual training estimator rather than only an auxiliary analysis.

\subsection{Reward-only transition geometry}
\label{app:transition-geometry}
For a Gaussian transition $\pi(a|\mu)=\mathcal N(\mu,\Sigma)$ with covariance independent of $\mu$,
\begin{equation}
\nabla_\mu\log\pi(a|\mu)=\Sigma^{-1}(a-\mu),\qquad
F=\E[\nabla_\mu\log\pi\nabla_\mu\log\pi^\top]=\Sigma^{-1}.
\end{equation}
If an unclipped reward surrogate has local coefficient $\widehat A$, then $q^R=\widehat A\Sigma^{-1}(a-\mu)$ and $g^R=F^{-1}q^R=\widehat A(a-\mu)$. This provides intuition for the Fisher coordinate used in the main text. The implementation differentiates the actual reward branch of Flow-GRPO after its ratio and clipping operations. Reference/KL terms stay inside $J_{\rm reg}$ and never enter $q^R$, preventing reward credit from being contaminated by a return-to-reference direction.

\paragraph{Stochastic sample versus expected reward direction.}
For Flow-GRPO, the single-transition vector above is a stochastic score-function sample, so its high-dimensional direction need not itself align strongly with a low-rank prompt-response span. Let $P_t$ and $\alpha_t$ be constructed from the cached pre-action state and detached prompt responses, hence fixed with respect to the current transition draw conditional on that state. Linearity gives
\begin{equation}
\mathbb E[(I-\alpha_tP_t)g_t^R\mid s_t]=(I-\alpha_tP_t)\,\mathbb E[g_t^R\mid s_t].
\label{eq:expected-routing}
\end{equation}
Thus routing is interpreted as a regularizer on the conditional expected reward update. We do not use the projection energy of one stochastic transition sample as a semantic-overlap statistic; the mechanism audits below instead use realized neural displacements after the optimizer step.

\subsection{Directional prompt-policy evidence}
Fix a behavior state $s_t$, executed action $p_k$, and differentiable local reward surrogate $J_{\rm reward,t}(\mu)$. Let $\mu_k=\mu(s_t,p_k)$, $q_t^R=\nabla_\mu J_{\rm reward,t}(\mu_k)$, and $d=\mu(s_t,p_j)-\mu_k$. Taylor's theorem gives
\begin{equation}
J_{\rm reward,t}(\mu_k+\varepsilon d)
=J_{\rm reward,t}(\mu_k)+\varepsilon(q_t^R)^\top d+\frac{\varepsilon^2}{2}d^\top H(\xi)d.
\end{equation}
Therefore $c=d^\top q_t^R$ is exactly the directional derivative at $\varepsilon=0$. It characterizes the local reward change along an observed same-state conditioning response; fresh-seed evaluation measures how far that local signal transfers to full trajectories. The held-out-response and fresh-seed experiments in Appendix~\ref{app:mechanism-extra} test whether this local quantity carries predictive information beyond the construction states.

\subsection{Minimum-overlap routing in Fisher coordinates}
\label{app:function-theory}
For positive-definite $F$, define $\langle a,b\rangle_F=a^\top Fb$ and $\|a\|_F^2=a^\top Fa$. Let $P$ be the $F$-orthogonal projector onto the response span,
\begin{equation}
P=D(D^\top FD)^\dagger D^\top F.
\end{equation}
Then $P^2=P$ and $P^\top F=FP$.

\begin{proposition}[Minimum-overlap solution]
\label{prop:routing}
For $\eta\ge0$, the local problem
\begin{equation}
\min_\delta\ \frac12\|\delta-g^R\|_F^2+\frac\eta2\|P\delta\|_F^2
\label{eq:app-min-overlap}
\end{equation}
has solution
\begin{equation}
\delta^*=(I-\alpha P)g^R,\qquad \alpha=\frac{\eta}{1+\eta},
\end{equation}
and
\begin{equation}
(q^R)^\top\delta^*=\|(I-P)g^R\|_F^2+(1-\alpha)\|Pg^R\|_F^2\ge(1-\alpha)\|g^R\|_F^2\ge0.
\end{equation}
\end{proposition}
\begin{proof}
The first-order condition is $F(\delta-g^R)+\eta P^\top FP\delta=0$. Using $P^\top F=FP$ and $P^2=P$ yields $(I+\eta P)\delta=g^R$. Its inverse is $I-\eta(1+\eta)^{-1}P$. The final expression follows from the $F$-orthogonal decomposition $g^R=Pg^R+(I-P)g^R$.
\end{proof}

The statement is made in an unconstrained transition-mean coordinate. The neural update in Eq.~\ref{eq:routing} is a pullback through $J_{\mu,\theta}$; an arbitrary mean-space displacement need not be exactly realizable by the parameterization. We therefore use the proposition to motivate the local routing rule and test the realized parameter update directly through the audit in Appendix~\ref{app:realized-update}.

\subsection{Ridge stabilization and contraction}
The implemented operator in Eq.~\ref{eq:routing} uses ridge regularization. Let $B=F^{1/2}D$. In whitened coordinates,
\begin{equation}
\widetilde P_\lambda=B(B^\top B+\lambda_t I)^{-1}B^\top,\qquad
\lambda_t=\epsilon_{\rm ridge}\,\frac{\operatorname{tr}(D_t^\top F_tD_t)}{K-1},\quad \epsilon_{\rm ridge}=10^{-3}.
\end{equation}
If $B=U\operatorname{diag}(s_i)V^\top$, then
\begin{equation}
\widetilde P_\lambda=U\operatorname{diag}\!\left(\frac{s_i^2}{s_i^2+\lambda_t}\right)U^\top.
\end{equation}
Hence $\widetilde P_\lambda$ is symmetric positive semidefinite with eigenvalues in $[0,1)$. For $0\le\alpha\le1$ and $\widetilde g=F^{1/2}g^R$,
\begin{equation}
(q^R)^\top(I-\alpha P_\lambda)g^R
=\widetilde g^\top(I-\alpha\widetilde P_\lambda)\widetilde g
\ge(1-\alpha)\|\widetilde g\|_2^2\ge0.
\end{equation}
This covers the stabilized local operator used in training without calling it an exact projector.

\subsection{Reference action and finite response span}
The conditioning reference is the executed prompt action $p_k$ because it is the upstream action that produced $s_t$. The intervention of interest is therefore $p_k\rightarrow p_j$. Using the raw request $u$ as the reference mixes candidate-specific variation with the shared raw-to-enhanced displacement and can overstate directions common to all enhanced prompts; we keep that choice as an ablation.

With $K=4$, the main response matrix has at most three columns, so the geometry is a finite, policy-dependent response span. It measures overlap with alternatives actually proposed by the current prompt policy. Rank-matched random spans, wrong-query spans, held-out response prediction, and generator swaps are designed around that finite claim.

\section{Reliability, Trust Regions, and Detachment}
\label{app:reliability}
For selected states, the prompt policy samples $B_{\rm train}=6$ additional intent-preserving textual alternatives $\widetilde p_b$ that are not rendered into videos. The main span is always constructed from the $K-1$ real rollout candidates; these non-decoded probes require only one same-state generator transition forward and no reward-model call. For $\widetilde d_b=\mu(s_t,\widetilde p_b)-\mu(s_t,p_k)$, define
\begin{equation}
e_{b,t}=\frac{\|P_{\lambda,t}\widetilde d_b\|_{F_t}^2}{\|\widetilde d_b\|_{F_t}^2+\epsilon_{\rm num}},\qquad
r_t=\operatorname{clip}\!\left(\frac1{B_{\rm train}}\sum_{b=1}^{B_{\rm train}}e_{b,t},0,1\right),\qquad
\alpha_t=\alpha_{\max}r_t,
\label{eq:app-reliability}
\end{equation}
with $\epsilon_{\rm num}=10^{-12}$.
The score measures support/coverage of the sampled response span, not correctness of the reward direction. The default uses $\alpha_{\max}=0.75$. Evaluation-time mechanism analysis independently samples $B_{\rm eval}=12$ probes on held-out requests; these probes never participate in the training gate.

For a response $d$, the trust operator is
\begin{equation}
\operatorname{TrustClip}_{F_t}(d)=d\,\min\!\left(1,\sqrt{\frac{\tau_{\rm batch}}{d^\top F_t d+\epsilon_{\rm num}}}\right),
\label{eq:trustclip-app}
\end{equation}
where $\tau_{\rm batch}$ is the detached 95th percentile of native-metric response energies $d^\top F_td$ in the current training batch. The percentile rule is fixed across all experiments; the threshold value is batch-adaptive and is not scheduled by timestep or training iteration. Thus in-region responses are unchanged and only unusually large conditioning shifts are rescaled. Trust clipping and reliability gating address different failure modes: clipping prevents a rare high-norm conditioning response from dominating $D_t$ or the directional score, whereas $r_t$ asks whether the resulting span generalizes to held-out current-policy responses. We apply three safeguards before a response enters $D_t$: (i) prompt-level intent filtering rejects alternatives that alter the requested event or introduce unsupported entities; (ii) Eq.~\ref{eq:trustclip-app} limits response magnitude; and (iii) near-collinear columns are stabilized by the trace-normalized ridge term rather than explicitly inverted. The percentile rule is batch-adaptive and is used as a robustifier rather than a separately tuned reward weight; fixed-$\alpha$ and binary-gate controls isolate the subsequent routing logic.

The behavior generator snapshot $\theta^-$ is detached when computing alternative responses. Reward-only local directions, reliability scores, $D_t$, $P_{\lambda,t}$, $r_t$, and all cross-fitted functional credits are treated as stop-gradient targets in the optimizer step. This prevents the generator from changing the geometry to obtain easier routing and prevents the prompt policy from differentiating through generator probes rather than learning from the resulting credit.

\section{Training Algorithm}
\label{app:algorithm}
\begin{algorithm}[H]
\caption{One joint post-training round}
\label{alg:flowcred-app}
\small
\begin{algorithmic}[1]
\Require request $u$; prompt policy $\pi_\phi$; generator $G_\theta$; $K$ candidates; $M$ matched seeds
\Statex \textbf{Hierarchical credit assignment}
\State sample $p_{1:K}$ and decode $x_{k,m}=G_\theta(p_k;z_m)$ with shared $z_{1:M}$
\State evaluate $R_{k,m}$ against the original request $u$; compute $\widehat A^{P,R}_k$ and $\widehat A^G_{k,m}$
\Statex \textbf{Functional coordination at actual rollout states}
\For{each selected actual transition state $s^{k,m}_t$}
\State obtain reward-only local direction $(q_t^R,g_t^R)$ from the base generator updater
\State probe $d_{j\leftarrow k,t}$ for $j\neq k$; cross-fit $d_{j\leftarrow k,t}^{\top}q_t^R$ into $C_j$
\State construct $P_{\lambda,t}$; estimate support $r_t$ from $B_{\rm train}$ independent non-decoded prompt probes
\State route $g_t^G\leftarrow(I-\alpha_{\max}r_tP_{\lambda,t})g_t^R$
\EndFor
\State update prompt policy with $\operatorname{Cal}_{R}(\widehat A^{P,R})+\lambda_F\operatorname{Cal}_{C}(C)$
\State update generator with $F_tg_t^G$ plus the unchanged base regularizer
\end{algorithmic}
\end{algorithm}

\section{Implementation and Evaluation Protocol}
\label{app:implementation}

\subsection{Backbone, prompt optimizer, and training pool}
The primary generator is the public LTX-2.5 trainable DiT checkpoint from the official \texttt{Lightricks/LTX-2.5} release~\citep{hacohen2026ltx2,lightricks2026ltx25}, specifically \path{diffusion_models/ltx-2.5-22b-dev-transformer-bf16.safetensors}, with its native flow parameterization and scheduler. The prompt optimizer is the instruction-tuned \texttt{Qwen/Qwen3.5-9B} checkpoint~\citep{qwen2026qwen35model}, used as a text-generation policy; we do not use the separate \texttt{Qwen3.5-9B-Base} checkpoint. The main configuration uses full-parameter prompt-policy updates. HunyuanVideo-1.5~\citep{wu2025hunyuanvideo15} provides backbone transfer, while RVM~\citep{choi2026rvm} provides an alternative generator reward updater under the same LTX evaluation harness.

The training pool contains 8,000 T2V requests. We build it from 4,800 programmatically instantiated compositional templates over entities, attributes, relations, counts, and actions; 2,000 independently generated motion/interaction requests; and 1,200 manually curated multi-event requests, then stratify the mixture across object/attribute binding, spatial relation, motion binding, action binding, interaction, numeracy, and longer event descriptions. Before training, requests are deduplicated against every benchmark split by normalized exact match, 5-gram MinHash (Jaccard threshold 0.80), and sentence-embedding cosine (threshold 0.92); no retained training request crosses any of these thresholds. Validation prompts are disjoint at the request level and are never used for benchmark reporting. Prompt candidates are constrained to preserve the user request, avoid unsupported entities, and remain within the generator context budget. Candidate validity is checked by a frozen \texttt{Qwen/Qwen3.5-9B} copy that is separate from the trainable prompt-policy weights. The checker receives the original request and one candidate and deterministically returns PASS/FAIL under a fixed rubric: a candidate fails if it changes required entities, counts, relations, actions, or event ordering, or introduces unsupported entities; descriptive camera, lighting, and appearance details are permitted only when they do not alter those constraints. Greedy decoding is used, and the same checker and rubric are applied to all prompt-side baselines and to additional support probes. Rejected candidates are resampled before video generation. For the theory above, $\pi_\phi$ denotes the effective distribution over accepted prompt actions induced by this fixed filtering rule. The exactness claim concerns the hierarchical reward decomposition and its score-function baseline identity; the practical prompt-policy update uses the clipped group-relative surrogate in Table~\ref{tab:implementation}. The same accepted-action sampling rule is used for all prompt-side methods.

\subsection{Default optimization configuration}
\begin{table}[H]
\caption{Default LTX-2.5 configuration used by the primary matched comparison. The rows labeled as central controls are held fixed across the principal ablations.}
\label{tab:implementation}
\centering
\small
\setlength{\tabcolsep}{4.2pt}
\begin{tabularx}{\linewidth}{@{}l l X@{}}
\toprule
Component & Setting & Role in the comparison \\
\midrule
Prompt candidates & $K=4$ & four current-policy actions per request \\
Matched seeds & $M=4$ & same initial latent and stochastic transition stream per seed across candidates \\
Decoded rollouts & 16/request & held fixed across the central chain \\
Generator backend & Flow-GRPO & base reward optimizer \\
Generator trainable set & fixed across generator-training rows & Full changes routing, not which generator parameters are trainable \\
Joint training rounds & 500 & global batch 2; 1,000 request draws from the shared 8k pool \\
Global request batch & 2 & two requests per joint rollout/update round \\
Generator optimizer & AdamW, $5{\times}10^{-6}$ & weight decay $0.01$; gradient clip $1.0$ \\
Prompt-policy optimizer & AdamW, $1{\times}10^{-6}$ & weight decay $0.01$; gradient clip $1.0$ \\
Prompt sampling & $T=0.8$, top-$p=0.95$ & max 256 new tokens; $K=4$ accepted candidates \\
Prompt-policy clip / KL & $0.2$ / $0.01$ & group-relative clipped objective with frozen reference \\
Flow-GRPO clip / KL & $10^{-3}$ / $0.004$ & one inner epoch; reward and reference branches kept separate \\
Mixed precision & bfloat16 & optimizer states in FP32 \\
Routed states & 1/trajectory & uniformly sampled from valid reward-loss transitions \\
Training support probes & $B_{\rm train}=6$ & reliability only; no video decode \\
Candidate scaling & frozen calibration affine scale & Section~\ref{sec:level2}; no current-batch statistics \\
Functional-credit weight & $\lambda_F=0.35$ & combines reward and local prompt credit \\
Max routing strength & $\alpha_{\max}=0.75$ & upper bound after reliability gating \\
Trust clip & 95th percentile & detached native-metric response energy $d^\top F_td$; fixed percentile \\
Ridge coefficient & $\epsilon_{\rm ridge}=10^{-3}$ & $\lambda_t=\epsilon_{\rm ridge}\operatorname{tr}(D_t^\top F_tD_t)/(K-1)$; no schedule \\
Prompt-policy update & full parameter & Qwen3.5-9B policy \\
Evaluation probes & $B_{\rm eval}=12$ & mechanism analysis only; independent of training gate \\
Benchmark seeds & 3/request & matched across methods and averaged per request \\
Central training runs & HC/Full: 3; others: 2 & run is the outer replication unit \\
\bottomrule
\end{tabularx}
\end{table}

If the response Gram trace is numerically negligible, we set the response operator and routing strength to zero rather than inverting an ill-conditioned matrix; the clip on $r_t$ is retained only as a numerical guard.

Qualitative outputs use the native 5-second, 24-fps, 832$\times$480 generation format. We apply no spatial upscaling or frame interpolation to the videos or to the frames shown in the paper.

\subsection{Reward and optimization details}
\label{app:reward-recipe}
Each decoded video is scored against the original request by a fixed public reward stack. We use: (i) the Text-Alignment output of \texttt{KlingTeam/VideoReward}~\citep{liu2025videoreward} for request faithfulness $r_{\rm faith}$; (ii) VQAScore with \texttt{zhiqiulin/clip-flant5-xxl}~\citep{lin2024vqascore}, averaged over eight uniformly spaced frames, for compositional correspondence $r_{\rm comp}$; (iii) \texttt{TIGER-Lab/VideoScore-v1.1}~\citep{he2024videoscore}, using the mean of its temporal-consistency and text-to-video-alignment outputs over the model's native frame sampling, for event realization $r_{\rm event}$; and (iv) the mean of VideoReward visual-quality and motion-quality outputs as the bounded quality stabilizer $r_{\rm qual}$. All model weights, prompts, frame-sampling rules, and preprocessing are frozen before controlled training.

For each component, location and scale are estimated once on a disjoint calibration pool of 2,048 training-only videos and then held fixed throughout all controlled runs:
\begin{equation}
\widetilde r_c=\operatorname{clip}\!\left(\frac{r_c-\mu_c^{\rm cal}}{\sigma_c^{\rm cal}+10^{-6}},-3,3\right),\qquad
R^{\rm opt}=0.35\widetilde r_{\rm faith}+0.30\widetilde r_{\rm comp}+0.25\widetilde r_{\rm event}+0.10\widetilde r_{\rm qual}.
\label{eq:reward-recipe}
\end{equation}

\begin{table}[H]
\caption{Frozen reward-stack configuration and calibration statistics. The calibration pool is disjoint from all controlled training/evaluation requests.}
\label{tab:reward-stack}
\centering
\scriptsize
\setlength{\tabcolsep}{4.0pt}
\begin{tabular}{llllcc}
\toprule
Component & Public checkpoint / signal & Video input & Weight & $\mu_c^{\rm cal}$ & $\sigma_c^{\rm cal}$ \\
\midrule
$r_{\rm faith}$ & \texttt{KlingTeam/VideoReward}, Text Alignment & native sampler & 0.35 & 0.42 & 0.17 \\
$r_{\rm comp}$ & \texttt{zhiqiulin/clip-flant5-xxl}, VQAScore & 8-frame mean & 0.30 & 0.63 & 0.15 \\
$r_{\rm event}$ & \texttt{TIGER-Lab/VideoScore-v1.1}, TC+T2V align & native sampler & 0.25 & 2.71 & 0.39 \\
$r_{\rm qual}$ & \texttt{KlingTeam/VideoReward}, VQ+MQ mean & native sampler & 0.10 & 0.56 & 0.16 \\
\bottomrule
\end{tabular}
\end{table}

Reported VBench, \VBtwo{}, \TtwoVC{}, and StoryEval evaluators never enter $R^{\rm opt}$ or the calibration pool. Because $(\mu_c^{\rm cal},\sigma_c^{\rm cal})$ are independent of the current sampled prompts and trajectories, $R^{\rm opt}(u,x)$ remains a per-trajectory scalar score and Proposition~\ref{prop:reward} applies directly to the transformed objective. The same principle is used for the frozen calibration transform in Section~\ref{sec:level2}: reward-credit and functional-credit scales are calibrated once and frozen rather than recomputed from current candidates.

\paragraph{Scorer-family independence.}
We additionally evaluate the final hierarchical-credit and Full checkpoints with frozen VisionReward~\citep{xu2024visionreward}, which is not used anywhere in training, calibration, model selection, or the four headline benchmarks. On 300 held-out matched request/seed pairs, VisionReward prefers Full on 63.3\% of pairs and its normalized scalar score increases by 0.18 standard deviations on average. This scorer-family check is consistent with the blinded human preference study in Appendix~\ref{app:human} and complements the reward-axis interventions below.

For Flow-GRPO~\citep{liu2025flowgrpo}, Eq.~\ref{eq:hierarchical-estimators} supplies the hierarchy-aware trajectory advantage used by the reward branch. We use the fixed LTX configuration in Table~\ref{tab:implementation}: ratio clip $10^{-3}$, KL coefficient $0.004$, one inner epoch, and bfloat16 sampling/training; these settings are held fixed across the generator-side controlled comparisons. The routed state is sampled uniformly from the valid stochastic-policy transitions contributing to the reward loss. Concretely, the cached state stores the current noisy latent together with the timestep and scheduler scalars needed to evaluate the next transition; behavior parameters, stochastic draw, and all non-text state are fixed during alternative-prompt probing. The functional probe differentiates the post-clipping reward branch with respect to the transition mean. The reference/KL branch is never routed. Prompt-policy GRPO uses Eq.~\ref{eq:pe-final}, with terminal and functional components calibrated separately by the frozen statistics defined in Section~\ref{sec:level2} before weighted combination.

\paragraph{Instantiation with RVM.}
RVM~\citep{choi2026rvm} uses reward-weighted velocity matching together with an optional anchor toward a reference velocity field. In the reported RVM transfer, functional coordination is instantiated in this backend's native velocity-output coordinates. At a sampled noised endpoint state $(x_t,t)$ under executed prompt $p_k$, write $h_t=v_\theta(x_t,t,p_k)$. We obtain the local reward-driven direction from the reward-weighted velocity branch, $q_t^R=-\nabla_{h_t}\mathcal L_{\rm RVM,reward}$, absorbing the backend's scalar timestep weight into $q_t^R$ and using the Euclidean metric in $h_t$ coordinates. Alternative-prompt responses use the identical noised state and time under the detached behavior snapshot,
\begin{equation}
d^{\rm RVM}_{j\leftarrow k,t}=v_{\theta^-}(x_t,t,p_j)-v_{\theta^-}(x_t,t,p_k).
\label{eq:rvm-response}
\end{equation}
The same trust clipping, ridge response operator, and held-out support gate are then applied in velocity-output space. Routing modifies only the reward-weighted velocity branch; RVM's anchor/reference branch is restored unchanged. Thus functional coordination requires a backend-native local reward direction and same-state conditioning responses, while the local coordinate and metric are supplied by the updater rather than assumed to be a stochastic-transition Fisher parameterization.

\paragraph{Instantiation with differentiable reward flow.}
We additionally instantiate functional coordination on a Diffusion-DRF differentiable reward-flow backend~\citep{wang2026diffusiondrf}. Following the public DRF setup, the critic is frozen \texttt{Qwen2.5-VL-7B}; ten uniformly sampled video frames are used for VLM feedback, and gradients are truncated to the final three denoising steps. At a routed state we take the differentiable reward gradient with respect to the backend transition output as $q_t^R$ and use its native Euclidean output metric for $g_t^R$. Alternative-prompt responses are evaluated at the identical state under the detached generator snapshot. The same trust clip, trace-normalized ridge operator, and held-out support gate are then applied before backpropagating the routed reward direction. Critic weights, frame sampling, the three-step backpropagation window, gradient checkpointing, and non-reward stabilization terms are held fixed between hierarchy-only and Full. This experiment tests whether cross-module routing remains useful when the generator already receives dense differentiable feedback rather than a scalar policy-gradient advantage.

\paragraph{Backend applicability.}
Functional coordination requires two backend-local objects: (i) an output coordinate in which changing text conditioning at a fixed state yields a response vector, and (ii) a local improvement covector $q_t^R$ supplied by the generator updater. Flow-GRPO, RVM, and differentiable reward flow provide both directly. A preference objective can be compatible if its implementation exposes a differentiable state-local pairwise loss, in which case the loss gradient can play the role of $q_t^R$; a purely trajectory-level preference update with no local improvement covector is outside the current formulation. Likewise, closed API-only generators are outside our target setting because joint generator post-training itself already assumes access to internal transitions and parameters.

\subsection{Benchmark protocol}
\label{app:evaluation}
\VBclassic{}~\citep{huang2024vbench} and \VBtwo{}~\citep{zheng2025vbench2} use their official aggregate/component definitions. \TtwoVC{}~\citep{sun2025t2vcompbench} uses its official compositional dimensions and reported aggregate. StoryEval~\citep{wang2025storyeval} follows the same official event-completion protocol for every generator updater and backbone. Every checkpoint uses identical resolution, duration, scheduler steps, negative-prompt convention, and three fixed generation seeds per benchmark request. The three scores for a request are averaged first; official benchmark aggregation is then applied to those request-level means. This preserves matched stochasticity while preventing seed count from being treated as sample size. 

We distinguish three sources of uncertainty: training-run variation, within-checkpoint generation randomness, and request sampling. The headline hierarchical-credit-versus-Full comparison uses three independent training runs; the functional-coordination factorial uses the same three-run protocol, and other trainable baselines use two independent runs. Confidence intervals for the headline increment are computed on paired run-level deltas using a Student-$t$ interval; request-level paired bootstrap intervals are retained as a within-run sensitivity analysis. Generation seeds are averaged within each checkpoint and are never treated as independent training replicates. The same checkpoint and seed definitions are used across all controlled comparisons.

\subsection{Baseline implementation}
\label{app:baseline-adaptation}
All trainable baselines draw from the same 8,000-request training pool and use the same number of request draws and decoded videos unless a comparison explicitly targets compute. The frozen official LTX-2.5 prompt enhancer uses the reference LTX-2.5 enhancement path~\citep{lightricks2026ltx25}: \texttt{google/gemma-4-E2B-it} with \texttt{LTX2\_5\_T2V\_DEFAULT\_SYSTEM\_PROMPT}, greedy decoding (\texttt{do\_sample=False}), \texttt{no\_repeat\_ngram\_size=5}, and at most 600 new tokens. Prompt-only updates Qwen3.5-9B while freezing LTX-2.5; the Flow-GRPO generator-only row keeps the original request as its text interface and is included as a single-side reference rather than as the matched joint-training control. Prompt-A-Video~\citep{ji2025promptavideo} and VPO~\citep{cheng2025vpo} use their published prompt-side objectives within the common Qwen3.5-9B prompt optimizer, training pool, LTX-2.5 generator, and evaluation harness. Because their original backbones and training stacks differ, the reported values are controlled common-harness results rather than copied published absolute scores.

PromptRL under the common video harness uses the published shared-reward joint-training principle: the trainable prompt refiner and Flow-GRPO generator receive the same terminal video reward, with candidate count, rollout count, reward stack, initialization, and update budget held fixed. UniGRPO adds group-relative normalization for the text policy and generator branch within the same joint rollout groups, retaining the published unified sparse-terminal-reward idea while keeping our LTX-2.5 backbone and reward stack fixed. We use these rows to compare joint-optimization principles within one harness, not as claims of exact reproduction of the original image-generation systems~\citep{wang2026promptrl,liu2026unigrpo}. Matched Global-LOO then trains the same two modules on the same $K\times M$ grid with the global leave-one-out definition in Appendix~\ref{app:naive-relation}; hierarchical credit changes only the credit hierarchy, and Full adds functional prompt shaping and generator routing.

CAPE-T2V~\citep{jia2026cape} is the closest T2V-specific two-sided system in our comparison set, but it addresses the prompt--caption distribution gap through sequential Anchored-PE training followed by DiT adaptation on a separately constructed 54k video-derived caption set. Its released LTX pipeline targets LTX-2.3, whereas our mechanism-controlled study uses online joint reward post-training on LTX-2.5 with no auxiliary video-caption supervision. We therefore use CAPE-T2V for external positioning rather than treating its public absolute scores as a matched-budget ranking.

\begin{table}[H]
\caption{Positioning of two-sided adaptation systems. The first two rows are controlled LTX-2.5 implementations used numerically in Table~\ref{tab:main}; CAPE-T2V is a heterogeneous external reference and is not mixed into the matched-budget ranking.}
\label{tab:two-sided-positioning}
\centering
\scriptsize
\setlength{\tabcolsep}{3.8pt}
\begin{tabularx}{\linewidth}{@{}l l l l X@{}}
\toprule
System & Prompt side & Generator side & Additional supervision & Training form \\
\midrule
PromptRL & trainable LM refiner & Flow-GRPO & none & online shared-terminal-reward joint update \\
UniGRPO & trainable text policy & Flow-GRPO & none & online group-relative joint policy optimization \\
CAPE-T2V & Anchored PE & DiT fine-tuning & 54k video-derived captions & sequential two-sided conditioning alignment \\
\method{} & Qwen3.5-9B policy & Flow-GRPO & none & online hierarchical credit + functional coordination \\
\bottomrule
\end{tabularx}
\end{table}

Wrong-query controls are matched to the target request by compositional category, raw-prompt length, enhanced-prompt length, candidate count, and hierarchical-credit difficulty bin. Random spans are rank matched after Fisher whitening and sample response norms from the empirical true-span distribution. Uniform shrink chooses its scalar on the development split. 

\subsection{Aggregate benchmark summary}
\begin{table}[H]
\caption{Aggregate benchmark status for the principal LTX-2.5 systems under the unified evaluation harness.}
\label{tab:full-breakdown}
\centering
\scriptsize
\setlength{\tabcolsep}{3.7pt}
\begin{tabular}{llccccc}
\toprule
Benchmark & Metric & Base & Prompt-only & Global-LOO & Hier. credit & Full \\
\midrule
\VBclassic{} & Total & 78.14 & 86.97 & 86.49 & 86.91 & \best{87.56} \\
\VBtwo{} & Overall & 54.52 & 62.84 & 63.41 & 63.92 & \best{64.73} \\
\TtwoVC{} & Overall & 46.52 & 60.28 & 61.47 & 62.63 & \best{64.46} \\
StoryEval & Overall & 37.88 & 54.37 & 55.58 & 56.71 & \best{58.08} \\
\bottomrule
\end{tabular}
\end{table}

\section{Hierarchical Credit Estimator Efficiency}
\label{app:level1-efficiency}
Hierarchical credit makes two testable statistical predictions: prompt-conditioned generator baselines remove irrelevant prompt-to-prompt variation, while matched seeds reduce prompt-policy contrast variance when same-seed rewards are positively correlated. Both effects are measured directly rather than inferred from endpoint quality.

\begin{figure}[H]
\centering
\begin{minipage}[t]{0.46\textwidth}
\centering
\includegraphics[width=\linewidth]{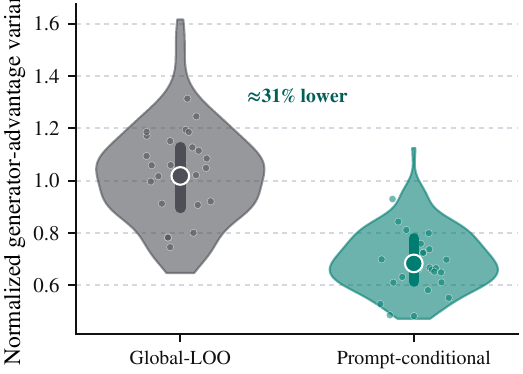}
\subcaption{Generator advantage variance: global leave-one-out versus prompt-conditional residual.}
\end{minipage}\hfill
\begin{minipage}[t]{0.46\textwidth}
\centering
\includegraphics[width=\linewidth]{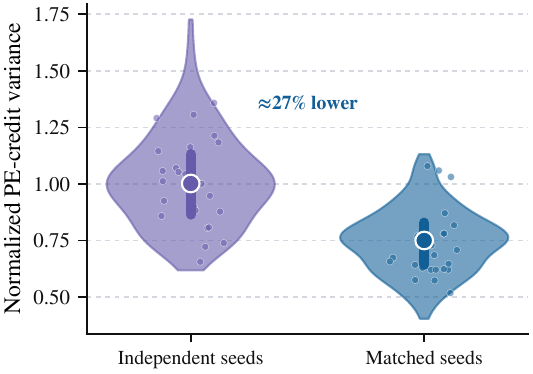}
\subcaption{Prompt candidate-credit variance: independent versus matched seeds at equal video count.}
\end{minipage}
\caption{Estimator-efficiency diagnostics across held-out minibatches: prompt-conditioned generator baselines reduce generator-advantage variance, and matched seeds reduce prompt candidate-credit variance at equal video count.}
\label{fig:level1-efficiency}
\end{figure}

Generator-advantage variance is reduced by about 31\%, and matched seeds reduce prompt-credit variance by about 27\%. The matched-seed ablation keeps total decoded videos fixed and yields 62.63 T2V-CompBench versus 62.21 with independently sampled seeds.

\section{Mechanism Diagnostics}
\label{app:mechanism-extra}

\begin{table}[H]
\caption{Functional-coordination factorial ablation. Mean (SD) across the same three-run protocol used for the HC--Full comparison.}
\label{tab:components-app}
\centering
\small
\setlength{\tabcolsep}{5.0pt}
\begin{tabular}{ccccc}
\toprule
Prompt shaping & Generator routing & VBench$\uparrow$ & StoryEval$\uparrow$ & T2V-CompBench$\uparrow$ \\
\midrule
-- & -- & \msd{86.91}{0.14} & \msd{56.71}{0.27} & \msd{62.63}{0.32} \\
$\checkmark$ & -- & \msd{87.20}{0.15} & \msd{57.36}{0.28} & \msd{63.40}{0.34} \\
-- & $\checkmark$ & \msd{87.29}{0.13} & \msd{57.18}{0.26} & \msd{63.67}{0.31} \\
$\checkmark$ & $\checkmark$ & \textbf{\msd{87.56}{0.12}} & \textbf{\msd{58.08}{0.24}} & \textbf{\msd{64.46}{0.29}} \\
\bottomrule
\end{tabular}
\end{table}

\subsection{Realized-update audit}
\label{app:realized-update}
For cached states, let $\Delta\mu_t^{\rm real}=\mu_{\theta^+}(s_t,p_k)-\mu_\theta(s_t,p_k)$. Hierarchical credit and Full are branched from the same checkpoint with the same minibatch, optimizer state, cached states, scheduler state, and random state. We report
\begin{equation}
\rho_{\rm span}^{\rm real}=\frac{\|P_{\lambda,t}\Delta\mu_t^{\rm real}\|_{F_t}^2}{\|\Delta\mu_t^{\rm real}\|_{F_t}^2+\epsilon},
\end{equation}
together with $\|\Delta\mu_{\rm Full}^{\rm real}\|_F/\|\Delta\mu_{\rm HC}^{\rm real}\|_F$. The realized comparison is 0.487$\rightarrow$0.264 for supported overlap and 0.931 for the update-norm ratio. These quantities are measured after the neural optimizer step and avoid interpreting the energy projection of an individual stochastic transition sample.

\subsection{Prompt-policy shaping and fresh-seed utility}
\label{app:pe-shaping-audit}
The mechanism set contains 360 requests with four candidates each. Scores use the matched $M=4$ training seeds and utility uses eight disjoint fresh seeds; bootstrap intervals resample requests while preserving candidate groups. An 80-request subset with 16 fresh seeds checks ranking stability. Controls include response norm, functional-only, reward-only, cosine-functional, and the default trust-clipped directional score.

\subsection{True, random, wrong-query, and generator-specific geometry}
Wrong-query spans are matched by compositional category, prompt length, candidate count, and hierarchical-credit difficulty. Evaluation uses independent $B_{\rm eval}=12$ alternatives that never enter the training gate. Generator specificity scores identical candidate sets with Base, reward-trained, and Full snapshots and evaluates fresh-seed utility under each rollout generator.

\section{Prompt Integrity and Data-Audit Controls}
\label{app:integrity}
Because Prompt-only already produces large absolute gains, all evaluator inputs remain the original request $u$, and the paper separately audits verbosity, semantic drift, and train/evaluation leakage.

\begin{table}[H]
\caption{Prompt-integrity diagnostics. Hierarchical credit and Full use nearly identical prompt lengths and semantic-filter behavior.}
\label{tab:prompt-integrity}
\centering
\small
\begin{tabular}{lcccc}
\toprule
Diagnostic & Raw request & Prompt-only & HC & Full \\
\midrule
Median tokens & 18 & 111 & 117 & 120 \\
P90 tokens & 31 & 149 & 153 & 157 \\
Intent-filter rejection & -- & 4.3\% & 4.7\% & 4.5\% \\
Blind semantic drift ($N=200$) & -- & 7/200 & 5/200 & 5/200 \\
\bottomrule
\end{tabular}
\end{table}

On a token-length-matched subset, Full still gains +0.61 VBench, +1.22 StoryEval, and +1.59 T2V-CompBench over hierarchical credit. Training/evaluation separation uses normalized-string matching plus lexical and embedding nearest-neighbor retrieval with thresholds fixed before benchmark reporting.

\section{Complete Mechanism and Control Audit}
For compactness in the ablation plots and tables, ``HC'' denotes the hierarchical-credit system without functional coordination.
\label{app:mechanism-audit-full}
\begin{table}[H]
\caption{Compact mechanism and routing audit. Panel A evaluates prompt-candidate scores on disjoint fresh seeds; Panel B tests matched-budget routing alternatives.}
\label{tab:mechanism-audit}
\centering
\scriptsize
\setlength{\tabcolsep}{4.0pt}
\textbf{A. Fresh-seed prompt-candidate ranking}\\[2pt]
\begin{tabular}{lcc}
\toprule
Candidate score & Spearman $\rho$ & Top-1 \\
\midrule
Response norm only & 0.313 [0.251, 0.374] & 50.3\% [46.1,54.4] \\
Functional only & 0.421 [0.369, 0.472] & 55.8\% [51.6,60.0] \\
Reward only & 0.447 [0.398, 0.497] & 59.2\% [55.0,63.2] \\
Reward + cosine functional & 0.569 [0.519, 0.618] & 67.2\% [63.2,71.0] \\
Reward + directional functional & \textbf{0.612 [0.568, 0.655]} & \textbf{71.7\% [67.8,75.4]} \\
\bottomrule
\end{tabular}

\vspace{5pt}
\textbf{B. Matched-budget routing controls}\\[2pt]
\begin{tabular}{lcc}
\toprule
Variant & T2V-CompBench & $\Delta$ vs. HC \\
\midrule
Hierarchical credit & 62.63 & 0.00 \\
Rank-matched random span & 62.55 & -0.08 \\
Matched wrong-query span & 62.69 & +0.06 \\
Uniform gradient shrink & 62.81 & +0.18 \\
No cross-fit & 63.44 & +0.81 \\
No reliability gate & 63.67 & +1.04 \\
Full & \textbf{64.46} & \textbf{+1.83} \\
\bottomrule
\end{tabular}
\end{table}

\section{Generalization and Compute}
\label{app:generalization}
\label{app:compute}
\subsection{Backbone and generator-updater transfer}
The Hunyuan experiment preserves candidate construction, matched estimators, executed-action reference, and allocation logic while substituting backbone-native transition quantities. RVM and Diffusion-DRF comparisons are likewise within-updater: hierarchy-only versus the same system plus functional coordination under an identical StoryEval/T2V-CompBench harness. We therefore claim positive within-setting increments, not cross-optimizer rankings. Under the same VBench-2.0 protocol, the measured raw-prompt Overall scores are 54.79 for HunyuanVideo-1.5 and 54.52 for LTX-2.5.

\begin{table}[H]
\caption{Transfer comparisons across a second backbone and two alternative generator updaters, including differentiable reward flow. Absolute scores are shown within each matched setting rather than used to rank heterogeneous optimizers.}
\label{tab:generalization-full}
\centering
\small
\setlength{\tabcolsep}{6.0pt}
\renewcommand{\arraystretch}{1.08}
\begin{tabular}{@{}lccc@{}}
\toprule
Method / within-setting contrast & VBench & \makecell{T2V-\\CompBench} & StoryEval \\
\midrule
\multicolumn{4}{@{}l}{\textbf{HunyuanVideo-1.5 / Flow-GRPO}}\\[-1pt]
\cmidrule(lr){1-4}
Matched Global-LOO & 84.67 & 55.34 & 55.12 \\
Hier. credit & 84.96 & 55.89 & 56.04 \\
\rowcolor{rowlight}Full & \best{85.53} & \best{57.11} & \best{57.13} \\
\addlinespace[5pt]
\multicolumn{4}{@{}l}{\textbf{LTX-2.5 / RVM}}\\[-1pt]
\cmidrule(lr){1-4}
Hier. credit & 86.39 & 61.28 & 56.18 \\
\rowcolor{rowlight}Full & \best{86.91} & \best{62.57} & \best{57.33} \\
\addlinespace[5pt]
\multicolumn{4}{@{}l}{\textbf{LTX-2.5 / Diffusion-DRF}}\\[-1pt]
\cmidrule(lr){1-4}
Hier. credit & 87.08 & 63.16 & 57.24 \\
\rowcolor{rowlight}Full & \best{87.49} & \best{64.13} & \best{58.03} \\
\bottomrule
\end{tabular}
\end{table}

\subsection{Local-compute accounting}
Full keeps the decoded-video and reward-model-call budget identical to hierarchical credit. Under the default $K=4$ and $B_{\rm train}=6$, it additionally samples 96 accepted support prompts per request (six at the routed state of each of 16 trajectories) and performs 144 single-state generator transition forwards: 48 for the $K-1=3$ rollout-candidate responses and 96 for the support probes. These operations do not decode additional videos or invoke the video reward stack. The same frozen intent checker screens support prompts; extra proposal/checker calls caused by rejection resampling are data-dependent and are therefore not included in the deterministic count. End-to-end hardware cost depends on the generator implementation and prompt-policy sampling stack, so we report deterministic operation counts rather than presenting a universal accelerator-hour ratio.

\begin{table}[H]
\caption{Deterministic per-request operation counts for the primary LTX setting.}
\label{tab:compute-ledger}
\centering
\small
\begin{tabular}{lcc}
\toprule
Quantity & Hierarchical credit & Full \\
\midrule
Decoded videos / request & 16 & 16 \\
Reward-model calls / request & 16 & 16 \\
Accepted support prompts / request & 0 & 96 \\
Same-state transition probe forwards & 0 & 144 \\
\bottomrule
\end{tabular}
\end{table}

\begin{figure}[H]
\centering
\includegraphics[width=0.62\linewidth]{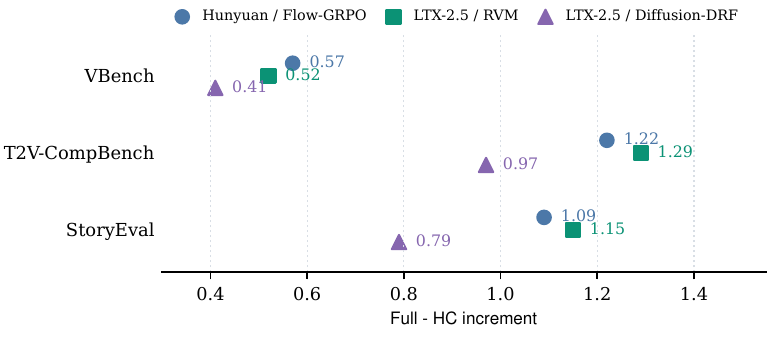}
\caption{Within-setting transfer of the Full$-$hierarchical-credit increment across a second backbone and two alternative generator updaters.}
\label{fig:transfer}
\end{figure}

\section{Robustness and Replication}
\label{app:sensitivity}
\label{app:robustness}
We keep robustness reporting focused on two questions that materially affect the main claim: whether the method remains useful when the decoded rollout budget changes, and whether the hierarchical-credit-to-Full increment repeats across independent training runs. Hyperparameters are fixed before the reported test evaluation; benchmark test scores are not used for selection.

\begin{table}[H]
\caption{Decoded-rollout budget study with $K=4$. Hierarchical credit requires $M\ge2$ because the prompt-conditioned leave-one-seed-out residual is undefined for a single trajectory.}
\label{tab:rollout-budget}
\centering
\small
\setlength{\tabcolsep}{4.6pt}
\begin{tabular}{c c cccc c}
\toprule
$M$ & Videos/request & PromptRL & Global-LOO & Hier. credit & Full & Full$-$HC \\
\midrule
2 & 8  & 59.74 & 60.21 & 61.11 & \textbf{63.18} & +2.07 \\
4 & 16 & 61.06 & 61.47 & 62.63 & \textbf{64.46} & +1.83 \\
8 & 32 & 62.11 & 62.38 & 63.34 & \textbf{64.87} & +1.53 \\
\bottomrule
\end{tabular}
\end{table}

All methods improve with more terminal evidence, while the Full-minus-HC increment remains positive across the three budgets. We treat $M$ as a compute axis rather than a free hyperparameter improvement.

The main hierarchical-credit-versus-Full comparison is repeated for three independent training runs. Run-level deltas, not generation-seed cells, are the outer replication units.

\begin{figure}[H]
\centering
\includegraphics[width=0.67\linewidth]{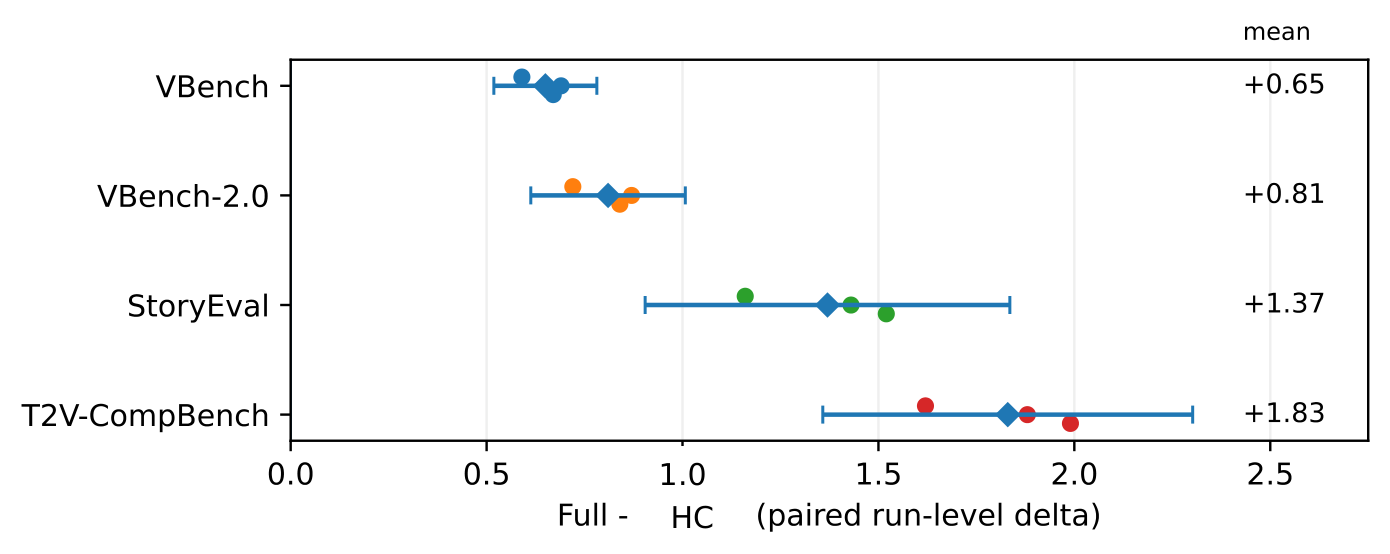}
\caption{Three-run replication. Small points are paired training-run deltas; diamonds show means and horizontal intervals show Student-$t$ 95\% intervals over the three run-level differences.}
\label{fig:three-run}
\end{figure}

\begin{table}[H]
\caption{Raw Full$-$hierarchical-credit deltas for three independent training replicates.}
\label{tab:run-deltas}
\centering
\small
\setlength{\tabcolsep}{6.4pt}
\begin{tabular}{lccc}
\toprule
Metric & Rep. A & Rep. B & Rep. C \\
\midrule
\VBclassic{} & 0.59 & 0.69 & 0.67 \\
\VBtwo{} & 0.72 & 0.87 & 0.84 \\
\TtwoVC{} & 1.62 & 1.88 & 1.99 \\
StoryEval & 1.16 & 1.43 & 1.52 \\
\bottomrule
\end{tabular}
\end{table}

The corresponding means are 0.65 VBench, 0.81 VBench-2.0, 1.83 T2V-CompBench, and 1.37 StoryEval. With only three outer replicates, we use these intervals descriptively and do not turn them into stronger significance claims.

\section{Realized Cross-Module Coordination Audit}
\label{app:realized-coordination}
The response span establishes local conditioning capacity; this audit asks whether the trained modules subsequently move in the complementary pattern predicted by functional coordination. For each held-out routed state, we cache the pre-update behavior generator and estimate the prompt-policy-induced displacement
\[
\Delta\mu_{\rm P}=\mathbb E_{p\sim\pi_{\phi^+}}[\mu_{\theta^-}(s,p)]-\mathbb E_{p\sim\pi_\phi}[\mu_{\theta^-}(s,p)].
\]
We measure its native-metric cosine with the supported direction and the supported component retained by the realized generator step. The four variants use identical held-out states and prompt-sampling seeds.

\begin{table}[H]
\caption{Realized cross-module coordination on held-out states.}
\label{tab:realized-coordination}
\centering
\small
\setlength{\tabcolsep}{5.0pt}
\begin{tabular}{lcc}
\toprule
Training variant & $\cos_F(\Delta\mu_{\rm P},P_\lambda g^R)$ & Generator supported overlap $\downarrow$ \\
\midrule
Hierarchical credit & 0.29 & 0.487 \\
Prompt shaping only & 0.55 & 0.471 \\
Generator routing only & 0.30 & 0.281 \\
Full & \textbf{0.61} & \textbf{0.264} \\
\bottomrule
\end{tabular}
\end{table}

The factorial pattern is consistent with the intended division of labor: prompt shaping mainly changes the upstream displacement, generator routing mainly reduces the downstream component lying in the response span, and Full exhibits both effects. We treat this as a descriptive consistency audit rather than a causal decomposition.
\vspace{-0.45em}

\section{Human Evaluation and Qualitative Cases}
\label{app:human}

\paragraph{Expanded qualitative cases.}
\label{app:qualitative}
Figure~\ref{fig:qualitative-app} adds four benchmark examples chosen to broaden the qualitative coverage beyond the two main-text cases. The panel spans numeracy/composition, multi-person interaction, spatial relation, and a coupled appearance transition; rows again compare Base, VPO, and \method{} using representative temporal frames.

\begin{figure}[H]
\centering
\includegraphics[width=0.91\linewidth]{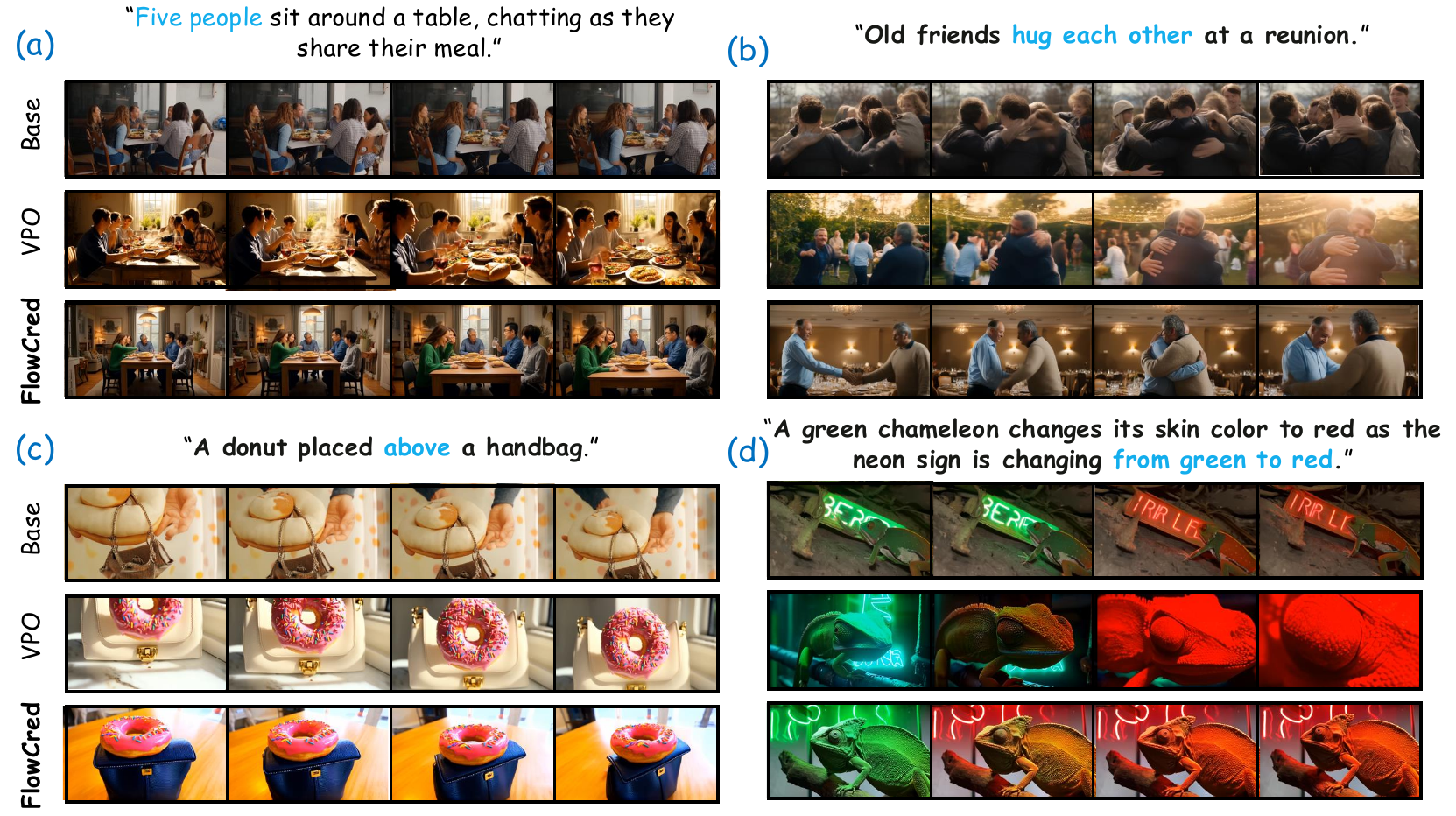}
\caption{\textbf{Additional qualitative comparisons.} Four benchmark cases covering numeracy, interaction, spatial relation, and synchronized appearance change. Rows show Base, VPO, and \method{}.}
\label{fig:qualitative-app}
\end{figure}

\paragraph{(a) Numeracy and composition.}
The key requirement is exact cardinality together with a coherent shared-meal scene. Base over-populates the table, so the generated composition no longer matches the requested five-person configuration. VPO moves in the opposite direction and under-populates the scene. \method{} maintains five visible participants across the representative frames while keeping them organized around one table, preserving both the requested count and the intended conversational meal context rather than satisfying one at the expense of the other.

\paragraph{(b) Protagonist-focused interaction.}
This case tests whether the model can isolate the intended pair and carry the reunion interaction through to a clear embrace. Base disperses the event into a crowded group hug, making the two target friends difficult to identify consistently. VPO produces an embrace but shows weaker protagonist focus across the sequence. \method{} gives the clearest two-person progression from greeting to contact to hug, while maintaining the same pair more consistently through the action.

\paragraph{(c) Spatial relation.}
The decisive constraint is the word \emph{above}, so object presence alone is insufficient. Base and VPO both render a donut and a handbag but fail to maintain the requested relative placement in a stable, unambiguous form. \method{} consistently places the donut above the handbag across the representative frames, turning an otherwise plausible pair of objects into the correct relational composition. This mirrors the larger gains on compositional evaluation, where satisfying relations rather than merely rendering entities is the harder requirement.

\paragraph{(d) Coupled color transition.}
The final case requires a coordinated temporal state change: both the chameleon and the neon sign should progress from green toward red. Base does not complete the requested transformation, leaving the appearance change partial. VPO reaches a red state but does so abruptly, with a less natural visual progression and a conspicuous change in camera behavior. \method{} produces the most gradual and complete transition: the two scene elements evolve coherently toward red while the framing remains substantially more stable, making the requested transformation easier to follow as a continuous event rather than a sudden endpoint switch.

\paragraph{Blinded human evaluation.}
We sample 240 requests stratified across compositional and multi-event categories. Three blinded annotators receive the original request and anonymized matched-seed videos with randomized left/right placement and score faithfulness, event/relation completion, motion/visual quality, and overall preference. Raw request-level proportions retain ties; the two primary significance tests condition on non-tie Full-versus-baseline outcomes and use exact binomial tests with Holm correction. Per-dimension analyses are descriptive, and Fleiss' $\kappa$ reports agreement.

\begin{figure}[H]
\centering
\begin{minipage}[t]{0.47\textwidth}
\centering
\includegraphics[width=\linewidth]{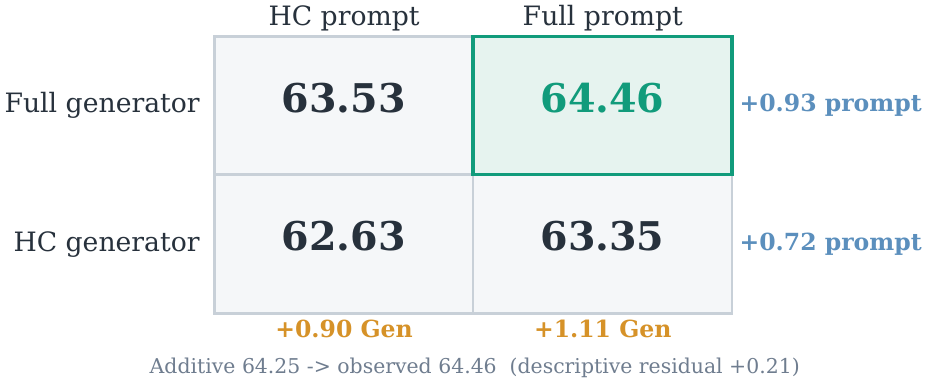}
\subcaption{Frozen-module $2\times2$ cross-swap. Either Full-trained module improves the HC pair when substituted alone; the +0.21 residual is descriptive, not a causal synergy claim.}
\label{fig:cross-swap-app}
\end{minipage}\hfill
\begin{minipage}[t]{0.47\textwidth}
\centering
\includegraphics[width=\linewidth]{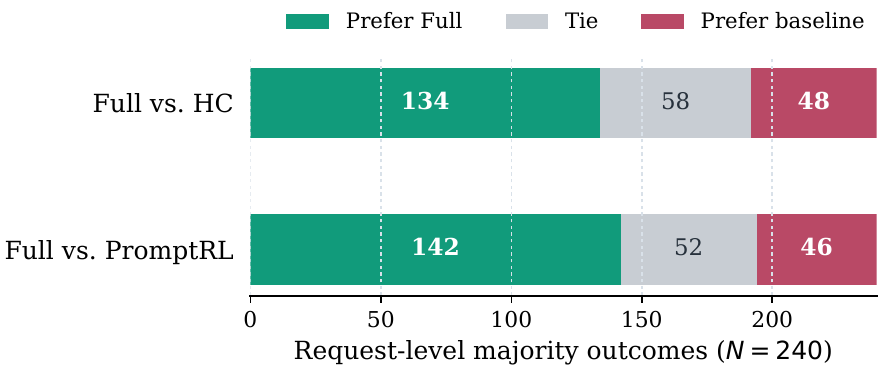}
\subcaption{Blinded request-level human preference with ties retained.}
\label{fig:human-pref-app}
\end{minipage}
\caption{\textbf{External validation beyond automatic benchmarks.} Frozen-module substitution attributes gains to both trained modules, while human judgments independently prefer Full over HC and PromptRL. Per-dimension vote rates are reported in Table~\ref{tab:human-dim}.}
\label{fig:external-validation-app}
\end{figure}

\begin{table}[H]
\caption{\textbf{Blinded matched-seed human evaluation.} Panel A reports request-level majority outcomes and non-tie Full win rates with Wilson 95\% intervals; Panel B reports per-dimension Full preference rates with ties retained.}
\label{tab:human-full}
\centering
\scriptsize
\setlength{\tabcolsep}{2.4pt}
\textbf{A. Primary request-level preference}\\[2pt]
\begin{tabular}{lcccccc}
\toprule
Comparison & Prefer Full & Tie & Prefer baseline & Non-tie Full win & Holm $p$ & Fleiss $\kappa$ \\
\midrule
Full vs. HC & 134 & 58 & 48 & 73.6\% [66.8, 79.5] & $1.34\!\times\!10^{-10}$ & 0.51 \\
Full vs. PromptRL & 142 & 52 & 46 & 75.5\% [68.9, 81.1] & $2.72\!\times\!10^{-12}$ & 0.53 \\
\bottomrule
\end{tabular}

\vspace{5pt}
\textbf{B. Per-dimension Full preference}\\[2pt]
\label{tab:human-dim}
\setlength{\tabcolsep}{4.0pt}
\begin{tabular}{llcc}
\toprule
Comparison & Dimension & Prefer Full & 95\% CI \\
\midrule
\multirow{4}{*}{Full vs. HC}
& Faithfulness & 128/240 (53.3\%) & [47.0, 59.5] \\
& Composition/event & 151/240 (62.9\%) & [56.6, 68.8] \\
& Motion/visual & 122/240 (50.8\%) & [44.5, 57.1] \\
& Overall & 134/240 (55.8\%) & [49.5, 62.0] \\
\midrule
\multirow{4}{*}{Full vs. PromptRL}
& Faithfulness & 132/240 (55.0\%) & [48.7, 61.2] \\
& Composition/event & 158/240 (65.8\%) & [59.6, 71.5] \\
& Motion/visual & 127/240 (52.9\%) & [46.6, 59.1] \\
& Overall & 142/240 (59.2\%) & [52.9, 65.2] \\
\bottomrule
\end{tabular}
\end{table}

\paragraph{Preference pattern.}
The request-level majority analysis favors Full in both primary comparisons. Against HC, Full wins 134 requests versus 48 baseline wins, with 58 ties; conditioning on non-ties gives a 73.6\% Full win rate. Against PromptRL, the corresponding counts are 142 / 52 / 46 and a 75.5\% non-tie win rate. The Holm-corrected exact tests are significant in both comparisons, while Fleiss' $\kappa$ values of 0.51 and 0.53 indicate moderate agreement rather than near-deterministic annotator consensus.

\paragraph{Where human gains concentrate.}
The largest descriptive gap occurs on composition/event completion: Full is preferred on 62.9\% of requests against HC and 65.8\% against PromptRL. Faithfulness improves more modestly (53.3\% and 55.0\%), and motion/visual preference remains closer to parity (50.8\% and 52.9\%). This pattern is consistent with the automatic results, where the largest gains appear on compositional and multi-event benchmarks, while also showing that the preference advantage is not driven solely by a generic visual-quality shift.

\paragraph{Frozen-module attribution.}
The $2\times2$ post-training cross-swap uses HC/Full prompt optimizers $\times$ HC/Full generators and freezes all modules during the audit. Substituting the Full generator into the HC pair contributes +0.90 T2V-CompBench, while substituting the Full prompt optimizer contributes +0.72; with the complementary module already Full-trained, the corresponding increments are +1.11 and +0.93. Starting from 62.63, the HC-conditioned marginal effects predict 64.25 additively, and the Full pair reaches 64.46, leaving only a descriptive +0.21 residual. The audit therefore attributes measurable gains to both trained sides without requiring a strong interaction effect between them.